\documentclass{article}

\usepackage{microtype}
\usepackage{graphicx}
\usepackage{subfigure}
\usepackage{booktabs} 
\usepackage{hyperref}

\usepackage{tikz}
\usepackage{pgfplots}
\pgfplotsset{compat=1.7}

\usepackage[accepted]{icml2022}
\usepackage{multirow}
\usepackage{amsmath, amsthm, amssymb}
\usepackage{cleveref}
\usepackage{adjustbox}

\newtheorem{prop}{Proposition}
\newtheorem{definition}{Definition}
\newcommand{\minimize}{\mathop{\rm minimize}}
\newcommand{\maximize}{\mathop{\rm maximize}}
\newcommand{\argmin}{\mathop{\rm argmin}}

\newcommand{\biborder}[1]{}

\usetikzlibrary{positioning, decorations.pathreplacing, decorations.markings, calc,arrows, pgfplots.groupplots, decorations.pathreplacing, decorations.markings, patterns}

\icmltitlerunning{Dataset Condensation via Efficient Synthetic-Data Parameterization}

\begin{document}
\twocolumn[
\icmltitle{Dataset Condensation via Efficient Synthetic-Data Parameterization}

\begin{icmlauthorlist}
\icmlauthor{Jang-Hyun Kim}{snu}
\icmlauthor{Jinuk Kim}{snu}
\icmlauthor{Seong Joon Oh}{naver,tb}
\icmlauthor{Sangdoo Yun}{naver}
\icmlauthor{Hwanjun Song}{naver}
\icmlauthor{Joonhyun Jeong}{face}
\icmlauthor{Jung-Woo Ha}{naver}
\icmlauthor{Hyun Oh Song}{snu}
\end{icmlauthorlist}

\icmlaffiliation{snu}{Department of Computer Science and Engineering, Seoul National University}
\icmlaffiliation{naver}{NAVER AI Lab}
\icmlaffiliation{face}{Image Vision, NAVER Clova}
\icmlaffiliation{tb}{University of Tübingen}

\icmlcorrespondingauthor{Hyun Oh Song}{hyunoh@snu.ac.kr}
\icmlkeywords{Machine Learning, Dataset Distillation, Classification, Continual learning}

\vskip 0.3in
]
\printAffiliationsAndNotice{}

\begin{abstract}
The great success of machine learning with massive amounts of data comes at a price of huge computation costs and storage for training and tuning. Recent studies on dataset condensation attempt to reduce the dependence on such massive data by synthesizing a compact training dataset. However, the existing approaches have fundamental limitations in optimization due to the limited representability of synthetic datasets without considering any data regularity characteristics. To this end, we propose a novel condensation framework that generates multiple synthetic data with a limited storage budget via efficient parameterization considering data regularity. We further analyze the shortcomings of the existing gradient matching-based condensation methods and develop an effective optimization technique for improving the condensation of training data information. We propose a unified algorithm that drastically improves the quality of condensed data against the current state-of-the-art on CIFAR-10, ImageNet, and Speech Commands. 
\end{abstract}

\section{Introduction}
\label{intro}

Deep learning has achieved great success in various fields thanks to the recent advances in technology and the availability of massive real-world data \citep{deeplearning}. However, this success with massive data comes at a price: huge computational and environmental costs for large-scale neural network training, hyperparameter tuning, and architecture search \citep{carbon,gpt,autoaugment,nas}.

An approach to reduce the costs is to construct a compact dataset that contains sufficient information from the original dataset to train models. A classic approach to construct such a dataset is to select the \textit{coreset} \citep{phillips2016coresets}. However, selection-based approaches have limitations in that they depend on heuristics and assume the existence of representative samples in the original data \citep{dsa}. To overcome these limitations, recent studies, called \textit{dataset condensation} or \textit{dataset distillation}, propose to synthesize a compact dataset that has better storage efficiency than the coresets \citep{dd}. The synthesized datasets have a variety of applications such as increasing the efficiency of replay exemplars in continual learning and accelerating neural architecture search \citep{dc}. 
\begin{figure}[!t]
    \centering
\def\rot{0}
\def\dd{0.3}
\def\ddd{0.2}
\def\channelwidth{6}
\resizebox{1.0\columnwidth}{!}{
\pgfdeclarelayer{fig}\pgfdeclarelayer{bg}
\pgfsetlayers{fig,bg}  
\begin{tikzpicture}[decoration={brace}][scale=1,every node/.style={minimum size=1cm},on grid]

\begin{pgfonlayer}{fig}
\foreach \i in {0,1,2,3,4,5,6} {
    \begin{scope}[
        xshift={.8*\channelwidth*\i},yshift={-.6*\channelwidth*\i},every node/.append style={
                yslant=0.2,xslant=0,rotate=\rot},yslant=0.2,xslant=0,rotate=\rot
            ]              
            \coordinate (A1\i) at (0,0);
            \coordinate (A3\i) at (3.5,-.7);
            \ifnum \i>0
            \ifnum \i=1 \def\cl1{purple} \fi
            \ifnum \i=2 \def\cl1{yellow} \fi
            \ifnum \i=3 \def\cl1{green} \fi
            \ifnum \i<4
            \draw[fill,\cl1!20,opacity=0.9] (A1\i) rectangle ++(1,1);
            \draw (A1\i) rectangle ++(1,1);
            \fi
            \fi

            \ifnum \i<6 \def\cl2{green} \fi
            \ifnum \i<4 \def\cl2{yellow} \fi
            \ifnum \i<2 \def\cl2{purple} \fi
            \ifnum \i<6
            \draw[fill,\cl2!20,opacity=0.9] (A3\i) rectangle ++(1,1);
            \draw (A3\i) rectangle ++(1,1);
            \fi
            \ifnum \i=4
            \draw [pattern=north west lines, pattern color=red!70!black!100]
                (A34) rectangle ++(1,1);
            \fi

            \coordinate (A5\i) at (3.5, 1.75);
            \ifnum \i=3
            \node[yslant=0, xslant=0, rotate=135] () at ($(A52)+(.38, .5)$) {...};
            \fi
            \ifnum \i<2
            \draw[fill,gray!20,opacity=0.9] (A5\i) rectangle ++(1,1);
            \draw (A5\i) rectangle ++(1,1);
            \fi
            \ifnum \i>4
            \draw[fill,gray!20,opacity=0.9] (A5\i) rectangle ++(1,1);
            \draw (A5\i) rectangle ++(1,1);
            \fi
            
    \end{scope}
}

\begin{scope}[every node/.append style={
            yslant=0.2,xslant=0,rotate=\rot},yslant=0.2,xslant=0,rotate=\rot
        ]              
    \draw [line width=0.2pt,color=red!70!black!100, pattern=north west lines,
    pattern color=red!70!black!100] 
        ($(A13) + (.05, .45)$) rectangle ++(.5, .5);              
    \draw [line width=0.2pt,color=red!70!black!100] 
        ($(A13) + (.65, .05)$) rectangle ++(.3, .3);

    \foreach \i in {0,1,2,3,4,5,6,7,8,9,10,11} {
        \draw[line width=.7pt,color=red!70!black!100, dotted] ($(A35) + (0.02,0.06 + 0.08*\i)$) -- ($(A35) + (0.99,0.06 + 0.08*\i)$);
    }
    \foreach \i in {0,1,2,3} {
        \draw[line width=.7pt,color=red!70!black!100, dotted] 
        ($(A13) + (.65, .05) + (0.00,0.03 + 0.08*\i)$) -- ($(A13) + (.65, .05) + (0.30,0.03 + 0.08*\i)$);
    }
\end{scope}

\coordinate (box) at ($(A50) - (3.8,0.1)$);
\draw (box) rectangle ++(2.8, 1.2);
\draw ($(box) + (0.03, 0.03)$) rectangle ++(2.74, 1.14);

\def\div{0.75}
\def\ybor{0.4}
\draw [latex-, color=blue!30!black!120]($(box)+(.2,2*\ybor)$)--($(box)+(\div,2*\ybor)$);
\draw [stealth-, color=red!70!black!100]($(box)+(.2,\ybor)$)--($(box)+(\div,\ybor)$);
\node[text width=2cm, minimum height=.7cm,
align=left, font=\scriptsize] at ($(box)+(\div*.5+2.94*.5,\ybor)$){Backpropagation};
\node[text width=2cm, minimum height=.7cm,
align=left, font=\scriptsize] at ($(box)+(\div*.5+2.94*.5,2*\ybor)$){Forward Pass};

\def\yint{2.45}
\foreach \j in {0,1} {
\foreach \i in {0,1} {
\begin{scope}[  
    xshift={2.2*\channelwidth*\i},yshift={.9*\channelwidth*\i},every node/.append style={
            yslant=0.2,xslant=0,rotate=\rot},yslant=0.2,xslant=0,rotate=\rot
        ]              

        \def\a{6.5}
        \def\b{-1.4+\yint*\j}
        \def\c{black!80}
        \def\d{blue!20}
        \ifnum\i=0
        \def\x{.53}
        \def\y{.7}
        \else 
        \def\x{.2}
        \def\y{.22}
        \fi

        \coordinate (A2\i0) at (\a, \b);
        \coordinate (A2\i1) at (\a, \b + \y);
        \coordinate (A2\i2) at (\a + \x, \b + \y);
        \coordinate (A2\i3) at (\a + \x, \b);

        \ifnum\i=1
        \draw[fill,\d,opacity=0.5] (A200) -- (A210) -- (A211) -- (A201) -- cycle;
        \draw[line width=0.1mm, \c] (A200) -- (A210) -- (A211) -- (A201) -- cycle;
        \draw[fill,\d,opacity=0.5] (A203) -- (A213) -- (A212) -- (A202) -- cycle;
        \draw[line width=0.1mm, \c] (A203) -- (A213) -- (A212) -- (A202) -- cycle;
        \draw[fill,\d,opacity=0.5] (A200) -- (A210) -- (A213) -- (A203) -- cycle;
        \draw[line width=0.1mm, \c] (A200) -- (A210) -- (A213) -- (A203) -- cycle;
        \draw[fill,\d,opacity=0.5] (A201) -- (A211) -- (A212) -- (A202) -- cycle;
        \draw[line width=0.1mm, \c] (A201) -- (A211) -- (A212) -- (A202) -- cycle;
        \fi

        \draw[fill,\d,opacity=.5] (A2\i0) rectangle ++(\x,\y);
        \draw[line width=0.1mm, \c] (A2\i0) rectangle ++(\x,\y);

\end{scope}
}
}

\node[text width=1.3cm, minimum height=.7cm, rounded corners=3pt, 
fill=blue!20!gray!20, inner sep=.05em, 
align=center, font=\scriptsize] at ($(A200)+(1.7,-0.89)$){};
\node[text width=1.3cm, minimum height=.7cm,
align=center, font=\scriptsize] at ($(A200)+(1.7,-0.785)$){Matching};
\node[text width=1.3cm, minimum height=.7cm,
align=center, font=\scriptsize] at ($(A200)+(1.7,-1.015)$){loss $D$};

\node[text width=1.3cm, minimum height=.7cm,
align=center, font=\scriptsize] at ($(A200)+(0.275,-0.75)$){Classifier};
\node[text width=1.3cm, minimum height=.7cm,
align=center, font=\scriptsize] at ($(A200)+(0.275,-1.)$){(shared)};

\node[text width=2cm, minimum height=.7cm,
align=left, font=\scriptsize] at ($(A30)+(0.35,-0.55)$){Synthetic};
\node[text width=2cm, minimum height=.7cm,
align=left, font=\scriptsize] at ($(A30)+(0.35,-0.8)$){training data};

\node[text width=2cm, minimum height=.7cm,
align=left, font=\scriptsize] at ($(A30)+(0.35,-0.55 + \yint)$){Original};
\node[text width=2cm, minimum height=.7cm,
align=left, font=\scriptsize] at ($(A30)+(0.35,-0.8 + \yint)$){training data};

\node[text width=2cm, minimum height=.7cm,
align=left, font=\scriptsize] at ($(A10)+(1.,-0.675)$){Condensed data};

\def\ya{-0.95 + 6.4 * 0.2}
\def\yb{-1.05 + 6.4 * 0.2}
\def\yc{-1.15 + 6.4 * 0.2}

\draw [-latex, color=blue!30!black!120](1.7,\ya)--(3.3,\ya);
\draw [-latex, color=blue!30!black!120](5.7,\yb + \yint)--(6.3,\yb + \yint);
\draw [stealth-, color=red!70!black!100](1.7,\yc)--(3.3,\yc);

\draw [-latex, color=blue!30!black!120](5.55,\ya)--(6.3,\ya);
\draw [stealth-, color=red!70!black!100](5.55,\yc)--(6.3,\yc);

\draw [-latex, color=blue!30!black!120] (7.35, \ya) arc (270:360:.7);
\draw [stealth-, color=red!70!black!120] (7.35, \yc) arc (270:360:.9);
\draw [-latex, color=blue!30!black!120] (7.35, \yb + \yint) arc (90:0:.8);

\node [anchor=west] at (1.8, \ya + .55) {\scriptsize \textbf{\textit{Multi-}}};   
\node [anchor=west] at (1.8, \ya + .3) {\scriptsize \textbf{\textit{formation} $\boldsymbol{f}$}};   

\end{pgfonlayer}
\end{tikzpicture}
}
\vspace{-2.5em}
\caption{Illustration of the proposed dataset condensation framework with multi-formation. Under the fixed-size storage for condensed data, multi-formation synthesizes multiple data used to train models. We optimize the condensed data in an end-to-end fashion by using the differentiable multi-formation functions.} 
\vspace{-1em}
\label{fig:decoding}
\end{figure}
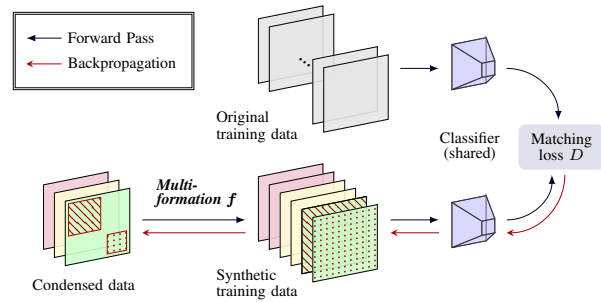

The natural data satisfy regularity conditions that form a low-rank data subspace \citep{huang1999}, \textit{e.g.}, spatially nearby pixels in a natural image look similar and temporally adjacent signals have similar spectra in speech \citep{speech_recog}. However, the existing condensation approaches directly optimize each data element, \textit{e.g.,} pixel by pixel, without imposing any regularity conditions on the synthetic data \citep{kip,dm}. Under the limited storage budget, this inefficient parameterization of synthetic datasets results in the synthesis of a limited number of data, having fundamental limitations on optimization. Furthermore, optimizing the synthetic data that have comparable training performance to the original data is challenging because it requires unrolling the entire training procedure. Recent studies propose surrogate objectives to address the challenge above, however, there are remaining questions on why certain objectives are better proxies for the true objective \citep{dm,dsa}.

In this work, we pay attention to making better use of condensed data elements and propose a novel optimization framework resolving the previous limitations. Specifically, we introduce a \textit{multi-formation} process that creates multiple synthetic data under the same storage constraints as existing approaches (\Cref{fig:decoding}). Our proposed process naturally imposes regularity on synthetic data while increasing the number of synthetic data, resulting in an enlarged and regularized dataset. In \Cref{theory}, we theoretically analyze the multi-formation framework and examine the conditions where the improvement is guaranteed. 
We further analyze the optimization challenges in the gradient matching method by \citet{dsa}  in \Cref{analysis}. Their approach induces imbalanced network gradient norms between synthetic and real data, which is problematic during optimization. Based on our analysis and empirical findings, we develop improved optimization techniques utilizing networks trained on the real data with stronger regularization and effectively mitigate the mentioned problems.

In this regard, we present an end-to-end optimization algorithm that creates information-intensive condensed data significantly outperforming all existing condensation methods. Given fixed storage and computation budgets, neural networks trained on our synthetic data show performance improvements of 10$\sim$20\%p compared to state-of-the-art methods in experimental settings with various datasets and domains including ImageNet and Speech Commands \cite{google_speech}. We further verify the utility of our condensed data through experiments on continual learning, demonstrating significant performance improvements compared to existing condensation and coreset methods. We release the source code at {\small\url{https://github.com/snu-mllab/Efficient-Dataset-Condensation}}.

\section{Preliminary}
\label{prelim}
Given the storage budget, the goal of data condensation is to build a surrogate dataset $\mathcal{S}$ of the original training dataset $\mathcal{T}$ such that an arbitrary model trained on $\mathcal{S}$ is similar to the one trained on $\mathcal{T}$ \citep{dd}. Oftentimes, the measure of similarity is in terms of the model performance on the test set because that leads to meaningful applications such as continual learning and neural architecture search \citep{dc}. Instead of solving this ultimate objective, previous methods have proposed different surrogates. For example, \citet{dd} propose to optimize $\mathcal{S}$ such that a model trained on $\mathcal{S}$ minimizes the loss values over $\mathcal{T}$. However, this approach involves a nested optimization with unrolling multiple training iterations, requiring expensive computation costs.\looseness=-1

Rather than direct optimization of model performance, \citet{dc} propose a simpler optimization framework that matches the network gradients on $\mathcal{S}$ to the gradients on $\mathcal{T}$. Let us assume a data point is $m$-dimensional and {\small$\mathcal{S}\in\mathbb{R}^{n\times m}$}, where $n$ is the number of data points in $\mathcal{S}$. \citet{dc} optimize the synthetic data as

\vspace{-1.8em}
{\small 
\begin{align}
\label{eq:matching}
    &\maximize\limits_{\mathcal{S}\in\mathbb{R}^{n\times m}}\ \sum_{t=0}^\tau \text{Cos}\left(\nabla_\theta\ell(\theta_t;\mathcal{S}), \nabla_\theta\ell(\theta_t;\mathcal{T})\right)\\[-0.25em]
    &\ \mathrm{subject}\ \mathrm{to}\ \ \theta_{t+1} = \theta_t - \eta \nabla_\theta\ell(\theta_t;\mathcal{S})\ \ \text{for}\ t = 0,\ldots,\tau-1, \nonumber
\end{align}
}%

\vspace{-0.7em}
where $\theta_t$ denotes the network weights at $t^{\mathrm{th}}$ training step from the randomly initialized weights $\theta_0$ given $\mathcal{S}$, {\small$\ell(\theta;\mathcal{S})$} denotes the training loss for weight $\theta$ and the dataset $\mathcal{S}$. {\small$\text{Cos}(\cdot,\cdot)$} denotes the channel-wise cosine similarity.
\citet{dc} have reported that the class-wise gradient matching objective is effective for dataset condensation. They propose an alternating optimization algorithm with the following update rules for each class $c$:

\vspace{-1.4em}
{\small
\begin{align*}
&S_c\leftarrow S_c + \lambda \nabla_{S_c} \text{Cos}\left(\nabla_\theta\ell(\theta;S_c), \nabla_\theta\ell(\theta;T_c)\right) \\[0.05em]
&\theta \leftarrow \theta - \eta \nabla_{\theta} \ell(\theta;\mathcal{S}), 
\end{align*}
}

\vspace{-1.6em}
where {\small$S_c$} and {\small$T_c$} denote the mini-batches from the datasets {\small$\mathcal{S}$} and {\small$\mathcal{T}$}, respectively. Under the formulation, \citet{dsa} propose to utilize differentiable siamese augmentation (DSA) for a better optimization of the synthetic data. DSA performs gradient matching on augmented data where the objective becomes {\small$\mathbb{E}_{\omega\sim\mathcal{W}}\left[\text{Cos}\left(\nabla_\theta\ell(\theta;a_\omega(\mathcal{S})), \nabla_\theta\ell(\theta;a_\omega(\mathcal{T}))\right)\right]$}. Here, $a_\omega$ means a parameterized augmentation function and $\mathcal{W}$ denotes an augmentation parameter space. Subsequently, \citet{dm} propose to match the hidden features rather than the gradients for fast optimization. However, the feature matching approach has some performance degradation compared to gradient matching \citep{dm}. 
Although this series of works have made great contributions, there are remaining challenges and questions on their surrogate optimization problems. 
In this work, we try to resolve the challenges by providing a new optimization framework with theoretical analysis and empirical findings.

\section{Multi-Formation Framework} \label{method}
In this section, we pay attention to the synthetic-data parameterization in optimization and present a novel data formation framework that makes better use of condensed data. We first provide our motivating observations and introduce a multi-formation framework with theoretical analysis.

\subsection{Observation}
We first provide our empirical observations on the effects of the number and resolution of the synthetic data in the matching problem.
The existing condensation approaches aim to synthesize a predetermined number of data about 10 to 50 per class \citep{dsa,kip}. The left subfigure in \Cref{fig:loss} shows the condensation matching loss curves of DSA over various numbers of synthetic data per class. As shown in the figure, more synthetic data lead to a smaller matching loss, indicating the importance of the number of synthetic data in the matching problem. For a comparison under the same data storage budget, we measure the matching loss on the same network after reducing the resolution of the optimized synthetic data and resizing the data to the original size. In the right subfigure in \Cref{fig:loss}, we find the resolution produces a moderate change in matching loss as the number of data does, even if we do not take the resolution modification into account during the condensation stage. For example, points at (16, 48) and (32, 12), which require an equal storage size, have similar loss values. Motivated by these results, we propose a multi-formation framework that makes better use of the condensed data and forms the increased number of synthetic data under the same storage budget.

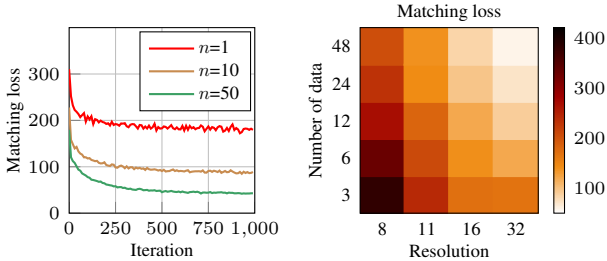
\begin{figure}[t]
    \begin{tikzpicture}[define rgb/.code={\definecolor{mycolor}{RGB}{#1}},
                    rgb color/.style={define rgb={#1},mycolor}]

\definecolor{gr}{RGB}{60,160,100}
\definecolor{or}{RGB}{200,140,80}
\pgfplotsset{
/pgfplots/colormap={hot2}{[1cm]rgb255(0cm)=(255,255,255) rgb255(2cm)=(240,140,20)
rgb255(5cm)=(160,0,0) rgb255(8cm)=(0,0,0)}
}

\begin{groupplot}[
        group style={columns=2, horizontal sep=1.4cm, 
        vertical sep=0.0cm},
        ]

\nextgroupplot[
            width=4.05cm,
            height=4.05cm,
            every axis plot/.append style={thick},
            grid=major,
            scaled ticks = false,
            xlabel near ticks,
            ylabel near ticks,
            tick pos=left,
            tick label style={font=\scriptsize},
            xlabel shift=-0.17cm,       
            ylabel shift=-0.18cm,
            label style={font=\scriptsize},
            xlabel style={align=center},
            xlabel={Iteration},
            ylabel={Matching loss},
            xmin=0,
            xmax=1000,
            ymax=400,
            ymin=0,
            xtick={0,250,500,750,1000},
            ytick={0,100,200,300},
            legend cell align=left,
            legend style={anchor=north east, nodes={scale=0.7}},
            ]

\addplot[red] table [y=ipc1, col sep=comma]{data/loss.csv};\addlegendentry{$n$=1}

\addplot[or] table [y=ipc10, col sep=comma]{data/loss.csv};\addlegendentry{$n$=10}

\addplot[gr] table [y=ipc50, col sep=comma]{data/loss.csv};\addlegendentry{$n$=50}

\nextgroupplot[
            width=3.95cm,
            height=4.05cm,
            colorbar,
            colorbar style={
                yticklabel style={
                    /pgf/number format/.cd,
                    fixed,
                    precision=0,
                    fixed zerofill,
                    },
                width=0.06*\pgfkeysvalueof{/pgfplots/parent axis width},
                ticklabel style={font=\scriptsize},
                ytick style={draw=none},
                },
            colorbar shift/.style={xshift=0.2cm},
            every axis plot/.append style={thick},
            title=Matching loss,
            title style={yshift=-0.25cm},
            xlabel=Resolution,
            ylabel=Number of data,
            enlargelimits=false,
            xlabel shift=-0.12cm,         
            ylabel shift=-0.1cm,
            style={font=\scriptsize},
            tick style={draw=none},
            label style={font=\scriptsize},
            symbolic x coords={8,11,16,32},
            symbolic y coords={48,24,12,6,3},
            ytick={48,24,12,6,3},
            scaled ticks = false,
            point meta min=50,
            point meta max=420,
            ]

\addplot [matrix plot,point meta=explicit] file {heat.dat};


\end{groupplot}
\end{tikzpicture}
    \vspace{-2.1em}
    \caption{(Left) Matching loss curves over an increasing number of synthetic data per class ($n$). (Right) Matching loss heat map over various resolutions and numbers of data per class. The x-axis refers to the downsampled image resolution. We measure values on the same network after resizing data to the original size (CIFAR-10).}
    \label{fig:loss}
    \vspace{-0.6em}
\end{figure}

\subsection{Multi-Formation}
The existing approaches directly match condensed data $\mathcal{S}$ to the original training data $\mathcal{T}$ and use $\mathcal{S}$ as the synthetic training data. Instead, we add an intermediate process that creates an increased number of synthetic data from $\mathcal{S}$ by mapping a data element in $\mathcal{S}$ to multiple data elements in the synthetic data (\Cref{fig:decoding}). The previous work by \citet{dsa} reports that the use of random augmentations in matching problems degrades performance due to the misalignment problem. They argue the importance of the deterministic design of the matching problem. In this regard, we propose to use a deterministic process rather than a random process.\looseness=-1

Consistent to existing approaches, we optimize and store condensed data {\small$\mathcal{S}\in\mathbb{R}^{n\times m}$}. For {\small$n'>n$}, we propose a multi-formation function {\small$f:\mathbb{R}^{n\times m}\to \mathbb{R}^{n'\times m}$} that augments the number of condensed data $\mathcal{S}$ and creates multiple synthetic training data {\small$f(\mathcal{S})$} in a deterministic fashion. 
For any matching objective $D$ (lower the better) and target task objective $\ell$, the optimization and evaluation stages of condensed data $\mathcal{S}$ with multi-formation function $f$ are
\vspace{-1.7em}

{\small
\begin{align*}
\mathcal{S}^* &= \argmin_{\mathcal{S}\in\mathbb{R}^{n\times m}} D(f(\mathcal{S}),\mathcal{T}) \tag*{(Optimization)} \\[-0.1em] 
\theta^* &= \argmin_\theta \ell(\theta; f(\mathcal{S}^*)). \tag*{(Evaluation)}
\end{align*}
}

\vspace{-1.2em}
That is, we perform matching on $\mathcal{T}$ using {\small$f(\mathcal{S})$} and use them for evaluation. This enables us to optimize the synthetic dataset with an increased number of data, using the same storage budget. \Cref{fig:decoding} illustrates the optimization process with multi-formation. Note, we can use conventional data augmentations following the multi-formation. 

Given a differentiable multi-formation function and matching objective, we optimize $\mathcal{S}$ in an end-to-end fashion by gradient descent. In this work, we design a simple differentiable multi-formation function and evaluate the effectiveness of our approach. The idea is to locally interpolate data elements while preserving the locality of natural data, \textit{i.e.}, spatially nearby pixels in a natural image look similar and temporally adjacent signals have similar spectra in speech \cite{huang1999, speech_recog}. Specifically, we partition each data and resize the partitioned data to the original size by using bilinear upsampling (\Cref{fig:formation}). Note, this formation function has negligible computation overhead. Furthermore, the formation function creates locally smooth synthetic data that might naturally regularize the optimization from numerous local minima. We use a fixed uniform partition function in our main experiments in \Cref{exp} and further analyze multi-scale and learnable formation functions in \Cref{appendix:formation}.\looseness=-1

\begin{figure}[t]
    \includegraphics[width=\columnwidth]{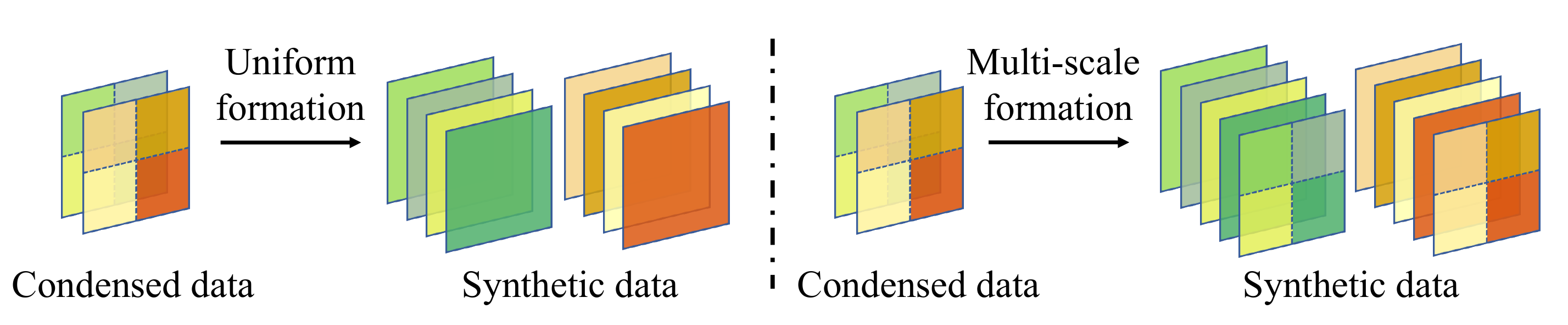}  
    \vspace{-2em}
    \caption{Illustration of the proposed multi-formation functions, in the case of multi-formation by a factor of 2.}
    \label{fig:formation}
    \vspace{-0.8em}
\end{figure}

\subsection{Theoretical Analysis}
\label{theory}
In this section, we aim to theoretically analyze our multi-formation framework. Here, we assume a data point is $m$-dimensional. The natural data have regularity that makes difference from random noise \citep{huang1999}. We assume that data satisfying this regularity form a subspace {\small$\mathcal{N}\subset \mathbb{R}^m$}. That is, the original training dataset  {\small$\mathcal{T}=\{t_i\}_{i=1}^{n_t}$} satisfies {\small$t_i\in\mathcal{N}$} for $i=1,\ldots,n_t$. 
With abuse of notation, we denote the space of datasets with $n$ data points as {\small$\mathbb{R}^{n\times m}=\{\{d_i\}_{i=1}^{n}\mid d_i\in \mathbb{R}^m\ \text{for}\  i=1,\ldots,n\}$}.
We further define the space of all datasets {\small$\mathcal{D}=$} $\cup_{n\in\mathbb{N}}${\small$\mathbb{R}^{n\times m}$} and the synthetic-dataset space of a multi-formation function {\small$f:\mathbb{R}^{n\times m}\to \mathbb{R}^{n'\times m}$}, {\small$\mathcal{M}_f=\{f(\mathcal{S})\mid \mathcal{S}\in\mathbb{R}^{n\times m}\}$}. We now introduce our definition of distance measure between datasets. We say data $d$ is closer to dataset {\small$X=\{d_i\}_{i=1}^{k}$} than $d'$, if {\small$\forall i\in[1,\ldots,k],\ \|d-d_i\|\le\|d'-d_i\|$}. 

\begin{definition}
A function {\small$D:\mathcal{D}\times\mathcal{D}\to [0,\infty)$} is a dataset distance measure, if it satisfies the followings:
{\small$\forall X, X'\in \mathcal{D}$} where {\small$X=\{d_i\}_{i=1}^k$}, {\small$\forall i \in[1,\ldots,k]$},
\vspace{-0.6em}
\begin{enumerate}\itemsep=0pt
    \item {\small$D(X,X)=0$} and {\small$D(X,X')=D(X',X)$}.
    \item {\small$\forall d\in \mathbb{R}^m\ s.t.\ d$} is closer to {\small$X'$} than {\small$d_{i},\ D(X\setminus\{d_{i}\}\cup{\{d\}}, X') \le D(X, X')$}.
    \item {\small$D(X, X'\cup\{d_{i}\})\le D(X, X')$}.
\end{enumerate}
\end{definition}

\vspace{-0.6em}
The definition above states reasonable conditions for dataset distance measurement. Specifically, the second condition states that the distance decreases if a data point in a dataset moves closer to the other dataset. The third condition states that the distance decreases if a data point in a dataset is added to the other dataset. Based on the definition, we introduce the following proposition. We provide the proof in \Cref{proof}.

\begin{prop}
    If {\small$\mathcal{N}^{n'}\subseteq \mathcal{M}_f$}, then for any dataset distance measure $D$,
    
    \vspace{-2.2em}
    {\small
    \begin{align*}
        \min_{\mathcal{S}\in \mathbb{R}^{n\times m}} D(f(\mathcal{S}),\mathcal{T}) \le \min_{\mathcal{S}\in \mathbb{R}^{n\times m}} D(\mathcal{S},\mathcal{T}).
    \end{align*}
    }
    \label{prop:bound}
\end{prop}

\vspace{-2em}
\Cref{prop:bound} states that our multi-formation framework achieves the better optimum, \textit{i.e.}, the synthetic dataset that is closer to the original dataset under any dataset distance measure. Note, the assumption {\small$\mathcal{N}^{n'}\subseteq \mathcal{M}_f$} means that the synthetic-dataset space by $f$ is sufficiently large to contain all data points in $\mathcal{N}$. In \Cref{theory:relaxed}, we provide theoretical results under a more relaxed assumption.

\section{Improved Optimization Techniques}
\label{analysis}
In this section, we develop optimization techniques for dataset condensation. We first analyze gradient matching \cite{dsa} and seek to provide an interpretation of why gradient matching on condensation works better than feature matching \citep{dm}. We then examine some of the shortcomings of existing gradient matching methods and propose improved techniques.

\subsection{Interpretation}
Convolutional or fully-connected layers in neural networks linearly operate on hidden features. 
From the linearity, it is possible to represent network gradients as features as in \Cref{prop1}. For simplicity, we consider one-dimensional convolution on hidden features and drop channel notations. 

\begin{prop}
Let {\small$w_t\in \mathbb{R}^{K}$} and {\small$h_t\in \mathbb{R}^{W}$} each denote the convolution weights and hidden features at the $t^\textrm{th}$ layer given the input data $x$. Then, for a loss function $\ell$, {\small$\frac{d \ell(x)}{d w_t} = \sum_i a_{t,i} h_{t,i}$}, where {\small$h_{t,i}\in \mathbb{R}^{K}$} denotes the $i^\text{th}$ convolution patch of {\small$h_t$} and {\small$a_{t,i}=\frac{d \ell(x)}{d w_t^\intercal h_{t,i}}\in \mathbb{R}$}.
\label{prop1}
\end{prop}

\vspace{-0.2em}
\Cref{prop1} states the gradients with respect to convolution weights can be regarded as the weighted sum of local features $h_{t,i}$. Note, the weight $a_{t,i}$ means the loss function sensitivity of the $i^\textrm{th}$ output at the $t^\textrm{th}$ layer, and we can interpret the network gradients as the saliency-weighted average local features. In this respect, we can view gradient matching as saliency-weighted average local feature matching.

Intuitively, saliency-weighting selectively extracts information corresponding to target labels. In addition, by matching averaged local features, we globally compare features regardless of location, which might be beneficial for datasets where target objects are non-aligned, \textit{e.g.}, ImageNet \citep{imagenet}. We conjecture these properties explain why gradient matching performs better than feature matching. In the following, we propose an improved gradient matching method by examining the shortcomings of existing gradient matching approaches.

\subsection{Problems and Solutions}
The existing gradient matching approach by DSA uses network weights $\theta_t$ trained on a condensed dataset $\mathcal{S}$ (see \Cref{eq:matching}). However, this approach has some drawbacks: 1) In the optimization process, $\mathcal{S}$ and $\theta_t$ are strongly coupled, resulting in a chicken-egg problem that generally requires elaborate optimization techniques and initialization \citep{em}. 2) Due to the small size of $\mathcal{S}$ ({\small$\sim1\%$} of the original training set), overfitting occurs in the early stage of the training and the network gradients vanish quickly. \Cref{fig:grad} shows that the gradient norm on $\mathcal{S}$ vanishes whereas the gradient norm on the real data $\mathcal{T}$ increases when the network is trained on $\mathcal{S}$. This leads to undesirable matching between two data sources, resulting in degraded performance when using distance-based matching objectives, such as mean squared error \citep{dc}. 

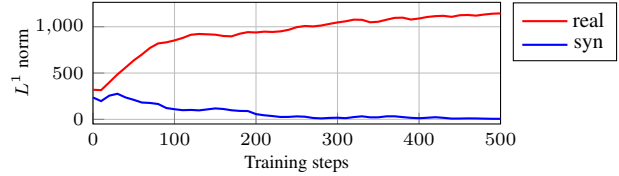
\begin{figure}[t]
    \begin{tikzpicture}[define rgb/.code={\definecolor{mycolor}{RGB}{#1}},
                    rgb color/.style={define rgb={#1},mycolor}]

\begin{groupplot}[
        group style={columns=1, horizontal sep=1.05cm, 
        vertical sep=0.0cm},
        ]

\nextgroupplot[
            width=7cm,
            height=3.2cm,
            every axis plot/.append style={thick},
            grid=major,
            scaled ticks = false,
            xlabel near ticks,
            ylabel near ticks,
            tick pos=left,
            tick label style={font=\scriptsize},
            xlabel shift=-0.1cm,         
            ylabel shift=-0.15cm,
            label style={font=\scriptsize},
            xlabel style={align=center},
            xlabel={Training steps},
            ylabel={$L^1$ norm},
            xmin=0,
            xmax=500,
            ymin=-50,
            legend style={at={(1.03,1)},anchor=north west},
            legend style={nodes={scale=0.8}}
            ]

\addplot[red] table [x=step, y=real, col sep=comma]{data/grad.csv};
\addplot[blue] table [x=step, y=syn, col sep=comma]{data/grad.csv};

\legend{real, syn}

\end{groupplot}
\end{tikzpicture}
    \vspace{-2.3em}
    \caption{Evolution of $L^1$ norm of the network gradients given real or synthetic data. The x-axis represents the number of training steps of the networks. Here, both networks are trained on the synthetic data with augmentations. We measure the values on CIFAR-10 with ConvNet-3 used in DSA.}
    \label{fig:grad}
    \vspace{-0.5em}
\end{figure}

To overcome these issues, we propose to utilize networks trained on $\mathcal{T}$ instead. By doing so, we optimize $\mathcal{S}$ with networks that are no longer dependent on $\mathcal{S}$, resulting in a decoupled optimization problem: 

\vspace{-1.9em}
{\small 
\begin{align}
    &\minimize\limits_{\mathcal{S}\in\mathbb{R}^{n\times m}}\ \bar{D}\left(\nabla_\theta\ell(\theta^\mathcal{T};f(\mathcal{S})), \nabla_\theta\ell(\theta^\mathcal{T};\mathcal{T})\right).
    \nonumber
\end{align}
}

\vspace{-2.1em}
Here, $\theta^\mathcal{T}$ represents network weights trained on $\mathcal{T}$ and $\bar{D}$ denotes a distance-based matching objective. In addition, the large size of $\mathcal{T}$ alleviates the gradient vanishing from overfitting \citep{bishop}. To further enhance the effect, we utilize stronger regularization for training networks. In detail, rather than a single random augmentation strategy adopted in DSA, we propose to use a sequence of augmentations and CutMix \citep{cutmix}. Note, the mixup techniques such as CutMix effectively resolve the neural networks' over-confidence issue by using soft labels for training \citep{puzzle,comix}. To sum up, the proposed utilization of real data and stronger augmentations effectively resolve the gradient vanishing problem and enable the use of distance-based objective functions, resulting in the better distillation of learning information onto the synthetic data.\looseness=-1

\subsection{Algorithm}
We further analyze the effect of network weights $\theta^\mathcal{T}$ on condensation. In detail, we examine when networks show the best condensation performance during the learning process on $\mathcal{T}$. Here, the performance means the test accuracy of neural networks trained on the condensed data. The left subfigure in \Cref{fig:acc} shows the performance of condensed data optimized by a network trained for a specific epoch. We observe the best condensation performance by the networks in the early phase of training near 10 epochs.

\begin{figure}[t]
    \begin{tikzpicture}[define rgb/.code={\definecolor{mycolor}{RGB}{#1}},
                    rgb color/.style={define rgb={#1},mycolor}]

\begin{groupplot}[
        group style={columns=1, horizontal sep=1.2cm, 
        vertical sep=0.0cm},
        ]

\nextgroupplot[
            width=4.05cm,
            height=4.05cm,
            every axis plot/.append style={thick},
            grid=major,
            scaled ticks = false,
            xlabel near ticks,
            ylabel near ticks,
            tick pos=left,
            tick label style={font=\scriptsize},
            xlabel shift=-0.1cm,         
            ylabel shift=-0.15cm,
            label style={font=\scriptsize},
            xlabel style={align=center},
            xlabel={Pretrained epoch},
            ylabel={Condensation acc.},
            xmin=0,
            xmax=100,
            ymin=40,
            ]

\addplot[blue, mark=*, mark size=0.6pt] table [x=epoch, y=acc, col sep=comma]{data/acc.csv};




\end{groupplot}
\end{tikzpicture}
    \begin{tikzpicture}
\pgfplotsset{
    scale only axis,
    xmin=0, 
    xmax=100,
    width=2.5cm,
    height=2.5cm,
    xlabel near ticks,
    ylabel near ticks,
    tick pos=left,
    tick label style={font=\scriptsize},
    xlabel shift=-0.1cm,         
    label style={font=\scriptsize},
    xlabel style={align=center},
    every axis plot/.append style={thick},
    grid=major,
    legend pos = north east,
    legend style={nodes={scale=0.7}}
}

\begin{axis}[
  axis y line*=right,
  ymin=0.2, 
  ymax=0.6,
  ytick={0.2,0.3,0.4,0.5},
  xlabel=Pretrained epoch,
  ylabel=Cosine similarity,
  ylabel shift=-0.15cm,  
]
\addplot[red] table [x=epoch, y=sim, col sep=comma, y index=1]{data/gradcls.csv}; \label{plot_one}
\end{axis}

\begin{axis}[
  axis y line*=left,
  axis x line=none,
  ymin=200, 
  ymax=1000,
  ytick={200,400,600,800},
  ylabel=$L^1$ norm,
  ylabel shift=-0.1cm,
]

\addlegendimage{/pgfplots/refstyle=plot_one}\addlegendentry{cos}

\addplot[blue] table [x=epoch, y=class, col sep=comma, y index=2]{data/gradcls.csv}; \addlegendentry{norm}

\end{axis}

\end{tikzpicture}
    \vspace{-2.3em}
    \caption{(Left) Condensation performance from fixed pretrained networks. The x-axis represents the number of epochs a network is trained on. (Right) Gradient analysis of the pretrained networks. The left axis measures the $L^1$ norm of the network gradients given a batch of data consisting of the same class. The right axis measures the average pairwise cosine-similarity between the gradients on a single data of the same class. The values are measured on ImageNet with 10 subclasses.}
    \label{fig:acc}
    \vspace{-1.2em}    
\end{figure}
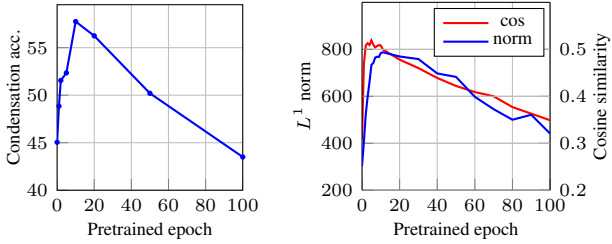

To clarify the observation, we measure the networks' gradient norm given an intra-class mini-batch (right subfigure in \Cref{fig:acc}). As a result, we find that the gradient norm increases in the early phase of training and then decreases during the further training epochs. 
We also observe a similar pattern when we measure pairwise cosine-similarity between the gradients given a single data of the same class. These results indicate the gradient directions among intra-class data coincide at the early phase of training but diverge as the training progresses. 
This phenomenon is similarly observed by \citet{hessian}; the first eigenvalue of the networks' hessian matrix increases in the early phase and decreases after a few epochs. Based on the observation, we argue that intra-class network gradients in the early training phase have more useful information to distill, and propose to utilize networks in the early training phase for condensation. Additionally, using the early phase neural networks has advantages in terms of the training cost. 

We empirically observe that using multiple network weights for condensation rather than the fixed network weights improves the generalization of the condensed data over various test models. Therefore, we alternately update $\mathcal{S}$ and $\theta^\mathcal{T}$ during the optimization process. In detail, we first initialize $\theta^\mathcal{T}$ by random initialization or loading pretrained weights trained only for a few epochs, and then we alternatively update $\mathcal{S}$ and $\theta^\mathcal{T}$. In addition, we periodically reinitialize $\theta^\mathcal{T}$ to maintain the network to be in the early training phase. 
Putting together with our multi-formation framework, we propose a unified algorithm optimizing information-intensive condensed data that compactly contain the original training data information. We name the algorithm as \textit{Information-intensive Dataset Condensation} (IDC) and describe the algorithm in \Cref{alg:matching}. Note, we adopt the siamese augmentation strategy by DSA.

\begin{figure}[t]
\vspace{-1.0em}
\begin{algorithm}[H]
    \caption{Information-Intensive Dataset Condensation}
    \label{alg:matching}
\begin{algorithmic}
\STATE {\bfseries Input:} Training data $\mathcal{T}$
\STATE {\bfseries Notation:} Multi-formation function $f$, 
parameterized augmentation function $a_\omega$, mixup function $h$, loss function $l$, number of classes $N_c$
\STATE {\bfseries Definition:} {\small$D(B,B';\theta) = \|\nabla_\theta \ell(\theta;B))-\nabla_\theta \ell(\theta;B') \|$}
\STATE Initialize condensed dataset $\mathcal{S}$
\REPEAT
\STATE Initialize or load pretrained network $\theta_1$
\FOR{$i=1$ {\bfseries to} $M$}
    \FOR{$c=1$ {\bfseries to} $N_c$}
        \STATE Sample an intra-class mini-batch {\small$T_c\sim \mathcal{T}, S_c\sim \mathcal{S}$} 
        \STATE Update {\small $S_c\leftarrow S_c - \lambda \nabla_{S_c} D(a_{\omega}(f(S_c)),a_{\omega}(T_c);\theta_i)$}
    \ENDFOR
    \STATE Sample a mini-batch \small$T\sim \mathcal{T}$
    \STATE Update \small $\theta_{i+1} \leftarrow \theta_i - \eta \nabla_{\theta} \ell(\theta_i;h(a_{\omega'}(T)))$
\ENDFOR
\UNTIL convergence
\STATE {\bfseries Output:} $\mathcal{S}$
\end{algorithmic}
\end{algorithm}
\vspace{-2.4em}
\end{figure}

\begin{table*}[!t]
\vspace{-0.6em}
\caption{Top-1 test accuracy of test models trained on condensed datasets from CIFAR-10. We optimize the condensed data using ConvNet-3 and evaluate the data on three types of networks. Pixel/Class means the number of pixels per class of the condensed data and we denote the compression ratio to the original dataset in the parenthesis. We evaluate each case with 3 repetitions and denote the standard deviations in the parenthesis. $\dagger$ denotes the reported results from the original papers. 
}
\vspace{0.5em}
\centering
{\small
\begin{tabular}{cl|ccccc|cc|c}
\toprule[1pt]
Pixel/Class & Test Model & Random & Herding & DSA & KIP & DM & IDC-I & IDC & Full dataset\\
\midrule                                   
\multirow{2}{*}{10$\times$32$\times$32} & ConvNet-3   & 37.2 & 41.7 &\ 52.1$^\dagger$ &\ 49.2$^\dagger$ & 53.8 & 58.3 (0.3) & \textbf{67.5} (0.5) & 88.1\\
\multirow{2}{*}{(0.2\%)}  & ResNet-10 & 34.1 & 35.9 & 32.9 & - & 42.3 & 50.2 (0.4) & \textbf{63.5} (0.1)	& 92.7\\
                    & DenseNet-121 & 36.5 & 36.7 & 34.5 & - & 39.0 & 49.5 (0.6) & \textbf{61.6} (0.6) & 94.2\\
\midrule                                   
\multirow{2}{*}{50$\times$32$\times$32} & ConvNet-3   & 56.5 & 59.8 &\ 60.6$^\dagger$ &\ 56.7$^\dagger$ & 65.6 & 69.5 (0.3) & \textbf{74.5} (0.1) & 88.1\\
\multirow{2}{*}{(1\%)}       & ResNet-10 & 51.2 & 56.5 & 49.7 & - & 58.6 & 65.7 (0.7) & \textbf{72.4} (0.5) & 92.7\\
                    & DenseNet-121 & 55.8 & 59.0 & 49.1 & - & 57.4 & 63.1 (0.2) & \textbf{71.8} (0.6) & 94.2\\
\bottomrule[1pt]
\end{tabular}
}
\label{tab:classification}
\vspace{-0.2em}
\end{table*}

\section{Experimental Results}
\label{exp}

In this section, we evaluate the performance of our condensation algorithm over various datasets and tasks. We first evaluate our condensed data from CIFAR-10, ImageNet-subset, and Speech Commands by training neural networks from scratch on the condensed data \citep{cifar, imagenet, google_speech}. Next, we investigate the proposed algorithm by performing ablation analysis and controlled experiments. Finally, we validate the efficacy of our condensed data on continual learning settings as a practical application \citep{cl_review}. We use multi-formation by a factor of 2 in our main experiments except for ImageNet where use a factor of 3. The other implementation details and hyperparameter settings of our algorithm are described in \Cref{imp_our}. We also provide experimental results on SVHN, MNIST, and FashionMNIST in \Cref{appendix:other}.

\subsection{Condensed Dataset Evaluation} 
\label{expsub:eval}
A common evaluation method for condensed data is to measure the test accuracy of the neural networks trained on the condensed data \citep{dsa}. It is widely known that test accuracy is affected by the type of test models as well as the quality of the data \citep{nas}. However, some previous works overlook the contribution from test model types and compare algorithms on different test models \citep{kip}. In this work, we emphasize specifying the test model and comparing the condensation performance on an identical test model for fair comparison. This procedure isolates the effect of the condensed data, thus enabling us to purely measure the condensation quality. 
We further evaluate the condensed data on multiple test models to measure the generalization ability of the condensed data across different architectures.

Baselines we consider are a random selection, Herding coreset selection \cite{herding}, and the previous state-of-the-art condensation methods; DSA, KIP, and DM \citep{dsa,kip,dm}. We downloaded the publicly available condensed data, and otherwise, we re-implement the algorithms following the released author codes. For the implementation details of the baselines, please refer to \Cref{imp_base}. We denote our condensed data with multi-formation as IDC and without multi-formation as IDC-I which can also be regarded as a method with the identity formation function. Finally, it is worth noting that KIP considers test models with ZCA pre-processing \citep{kip}. However, we believe test models with standard normalization pre-processing are much more common to be used in classification and continual learning settings \citep{autoaugment,vit,icarl}. In this section, we focus on test models with standard normalization pre-processing. For experimental results with ZCA, please refer to \Cref{appendix:zca}.

\vspace{-0.5em}
\paragraph{CIFAR-10.} 
The CIFAR-10 training set consists of 5,000 images per class each with $32\times 32$ pixels. Following the condensation baselines, we condense the training set with the storage budgets of 10 and 50 images per class by using 3-layer convolutional networks (ConvNet-3). For network architecture effect on condensation, please refer to \Cref{appendix:architecture}. We evaluate the condensed data on multiple test models: ConvNet-3, ResNet-10, and DenseNet-121 \citep{resnet,densenet}. It is worth noting that  \citet{dsa} used data augmentation when evaluating DSA but did not apply any data augmentation when evaluating simple baselines Random and Herding. This is not a fully fair way to compare the quality of data. In our paper, we re-evaluate all baselines including DSA by using the same augmentation strategy as ours and report the best performance for fair comparison. For the more detailed results on augmentation, please refer to \Cref{appendix:sa}.

\begin{table}[t]
\vspace{-0.65em}
\caption{Top-1 test accuracy of test models with the fixed training steps. Each row matches the same dataset storage size and evaluation cost. \textit{CN} denotes ConvNet-3, \textit{RN} denotes ResNet-10, and \textit{DN} denotes DenseNet-121. We measure training times on an RTX-3090 GPU.}
\vspace{0.5em}
\centering
\begin{adjustbox}{max width=\columnwidth}
\begin{tabular}{cl|ccc|cc|c}
\toprule[1pt]
Pixel & Test & \multirow{2}{*}{DSA} & \multirow{2}{*}{KIP} & \multirow{2}{*}{DM} & \multirow{2}{*}{IDC-I} & \multirow{2}{*}{IDC} & Evaluation \\
Ratio & Model &  &  &  &  &  & Time \\
\midrule                                   
\multirow{3}{*}{0.2\%} & CN  & 52.1 & 49.1 & 53.8 & 58.3 & \textbf{65.3} & 10s\\ 
                    & RN  & 32.9 & 40.8 & 42.3 & 50.2 & \textbf{57.7} & 20s\\ 
                    & DN  & 34.5 & 42.1 & 39.0 & 49.5 & \textbf{60.6} & 100s\\ 
\midrule                                   
\multirow{3}{*}{1\%} & CN  & 60.6 & 57.9 & 65.6 & 69.5 & \textbf{73.6} & 50s\\
                    & RN & 49.7 & 52.9 & 58.6 & 65.7 & \textbf{72.3} & 90s\\ 
                    & DN  & 49.1 & 54.4	& 57.4 & 63.1 & \textbf{71.6} & 400s\\ 
\bottomrule[1pt]
\end{tabular}
\end{adjustbox}
\label{tab:computation}
\vspace{-1em}
\end{table}

\begin{table*}[!ht]
\vspace{-0.6em}
\caption{Top-1 test accuracy of test models trained on condensed datasets from ImageNet-subset. We optimize the condensed data using ResNetAP-10 and evaluate the data on three types of networks. We evaluate the condensed data by using the identical training strategy.}
\vspace{0.5em}
\centering
{\small
\begin{tabular}{ccl|cccc|cc|c}
\toprule[1pt]
Class & Pixel/Class & Test Model & Random & Herding & DSA & DM & IDC-I & IDC & Full Dataset \\
\midrule                                   
\multirow{3}{*}{10} & \multirow{2}{*}{10$\times$224$\times$224} & ResNetAP-10  & 46.9 & 50.4	& 52.7 & 52.3 & 61.4 (0.8) & \textbf{72.8} (0.6) & 90.8\\
                    & \multirow{2}{*}{(0.8\%)} & ResNet-18  & 43.3 & 47.0	& 44.1 & 41.7 & 56.2 (1.2) & \textbf{73.6} (0.4) & 93.6\\
                    & & EfficientNet-B0  & 46.3 & 50.2	& 48.3 & 45.0 & 58.7 (1.4) & \textbf{74.7} (0.5) & 95.9\\
\midrule                                   
\multirow{3}{*}{10} & \multirow{2}{*}{20$\times$224$\times$224} & ResNetAP-10  & 51.8 & 57.5	& 57.4 & 59.3 & 65.5 (1.0) & \textbf{76.6} (0.4) & 90.8\\
                    & \multirow{2}{*}{(1.6\%)} & ResNet-18  & 54.3 & 57.9	& 56.9 & 53.7 & 66.0 (0.7) & \textbf{75.7} (1.0) & 93.6\\
                    & & EfficientNet-B0 & 60.3 & 59.0 & 62.5 & 57.7 & 66.3 (0.5) & \textbf{78.1} (1.0) & 95.9\\
\midrule                                   
\multirow{3}{*}{100} & \multirow{2}{*}{10$\times$224$\times$224} & ResNetAP-10  &  20.7 & 22.6 & 21.8 & 22.3 & 29.2 (0.4) & \textbf{46.7} (0.2) & 82.0\\
                    & \multirow{2}{*}{(0.8\%)} & ResNet-18  & 15.7 & 15.9 & 13.5 & 15.8 &	23.3 (0.3) & \textbf{40.1 }(0.5) & 84.6 \\
                    & & EfficientNet-B0  & 22.4 & 24.5 & 19.9 & 20.7 &	27.7 (0.6) & \textbf{36.3} (0.6) & 85.6\\
\midrule                                   
\multirow{3}{*}{100} & \multirow{2}{*}{20$\times$224$\times$224} & ResNetAP-10  & 29.7 & 31.1 & 30.7 & 30.4 &	34.5 (0.1) & \textbf{53.7} (0.9) & 82.0\\
                    & \multirow{2}{*}{(1.6\%)} & ResNet-18  & 24.3 & 23.4 & 20.0 & 23.4 & 29.8 (0.2) & \textbf{46.4} (1.6) & 84.6 \\
                    & & EfficientNet-B0  & 33.2 & 35.6 & 30.6 & 31.0 & 33.2 (0.5) & \textbf{49.6} (1.2) & 85.6\\
\bottomrule[1pt]
\end{tabular}
}
\label{tab:imgnet}
\vspace{-0.2em}
\end{table*}

\Cref{tab:classification} summarizes the test accuracy of neural networks trained on each condensed data. From the table, we confirm that both IDC and IDC-I significantly outperform all the baselines. Specifically, IDC outperforms the best baseline by over $10\%$p across all the test models and compression ratios. However, IDC requires additional training steps to converge due to the formation process in general. Considering applications where training cost matters, such as architecture search, we compare methods under the fixed training steps and report the results in \Cref{tab:computation}. That is, we reduce the training epochs when evaluating IDC, and match the number of gradient descent steps identical to the other baselines. In the case of KIP, which originally uses a neural tangent kernel for training networks, we re-evaluate the dataset by using stochastic gradient descent as others to match the computation costs. \Cref{tab:computation} shows IDC still consistently outperforms baselines by a large margin.

\vspace{-0.5em}
\paragraph{ImageNet.}
Existing condensation methods only perform the evaluation on small-scale datasets, such as MNIST or CIFAR-10. To the best of our knowledge, our work is the first to evaluate condensation methods on challenging high-resolution data, ImageNet \citep{imagenet}, to set a benchmark and analyze how the condensation works on large-scale datasets. We implement condensation methods on ImageNet-subset consisting of 10 and 100 classes \citep{imagenetsubset}, where each class consists of approximately 1200 images. We provide a detailed dataset description in \Cref{data:imagenet}. Note, KIP requires hundreds of GPUs for condensing CIFAR-10 and does not scale on ImageNet. In the ImageNet experiment, we use ResNetAP-10 for condensation, which is a modified ResNet-10 by replacing strided convolution as average pooling for downsampling \citep{dm}. For test models, we consider ResNetAP-10, ResNet-18, and EfficientNet-B0 \citep{efficientnet}.

\Cref{tab:imgnet} summarizes the test accuracy of neural networks trained on the condensed data. The table shows IDC and IDC-I significantly outperform all the baselines across the various numbers of classes, compression ratios, and test models. One of the notable results is that the existing condensation methods do not transfer well to other test models. For example, DM performs better on ResNetAp-10 compared to Random selection but performs poorly on other test models. On contrary, IDC consistently outperforms other methods regardless of test model types. This indicates that our networks trained on large real datasets extract more task-relevant information with less architectural inductive bias than randomly initialized networks (DM) or networks trained on synthetic datasets (DSA). In \Cref{fig:samples}, we provide representative condensed samples from IDC-I. Note, these samples are initialized by random real training samples. We provide the qualitative comparison of our condensed data and real training data in \Cref{appendix:visual}.

\begin{figure}[t]
    \centering
    \includegraphics[width=0.145\textwidth]{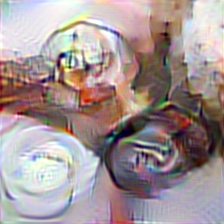}    
    \hspace{0.03em}
    \includegraphics[width=0.145\textwidth]{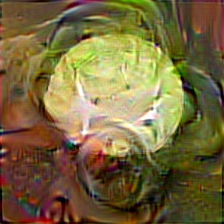}
    \hspace{0.03em}
    \includegraphics[width=0.145\textwidth]{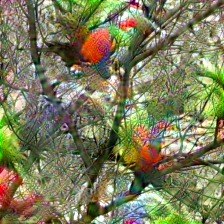}\\   
    \vspace{0.4em}
    \includegraphics[width=0.145\textwidth]{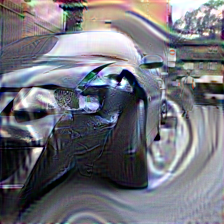}    
    \hspace{0.03em}
    \includegraphics[width=0.145\textwidth]{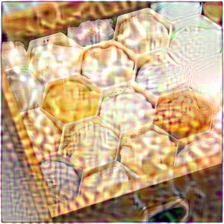}    
    \hspace{0.03em}
    \includegraphics[width=0.145\textwidth]{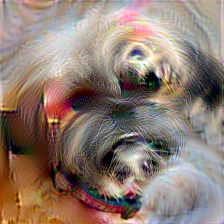}\\
    \vspace{0.4em}
    \hspace{0.04em}
    \includegraphics[width=0.145\textwidth]{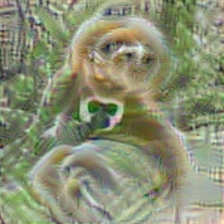}    
    \hspace{0.03em}
    \includegraphics[width=0.145\textwidth]{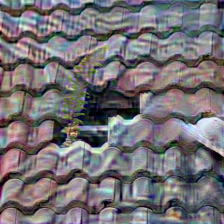}    
    \hspace{0.03em}
    \includegraphics[width=0.145\textwidth]{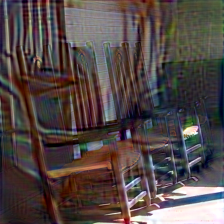}    
    \vspace{-0.5em}
    \caption{Representative samples from IDC-I condensed data on ImageNet-100. The corresponding class labels are as follows: bottle cap, cabbage, lorikeet, car wheel, honeycomb, Shih-Tzu, gibbon, tile roof, and rocking chair.}
    \vspace{-1em}
\label{fig:samples}
\end{figure}

\vspace{-0.5em}
\paragraph{Speech Domain.}
We evaluate our algorithm on speech domain data to verify the generality of our algorithm. In detail, we condense Mini Speech Commands that contains 8,000 one-second audio clips of 8 command classes \citep{google_speech}. We preprocess speech data and obtain magnitude spectrograms each of size $64\times 64$. For a detailed description of the dataset and preprocessing, please refer to \Cref{data:speech}. In the case of speech data, we use a one-dimensional multi-formation function by a factor of 2 along the time-axis of a spectrogram. \Cref{tab:speech} shows the test accuracy on the speech dataset. IDC consistently outperforms baseline methods by large margins and achieves test performance close to the full dataset training, verifying its effectiveness on speech domain as well as on image domain.\looseness=-1

\begin{table}[t]
\vspace{-0.6em}
\caption{Top-1 test accuracy of ConvNet-4 trained on condensed spectrograms. \textit{Rand.} and \textit{Herd.} denote Random and Herding.}
\vspace{0.5em}
\centering
\begin{adjustbox}{max width=\columnwidth}
\begin{tabular}{c|cccc|cc|c}
\toprule[1pt]
Spectrogram/  & \multirow{2}{*}{Rand.} & \multirow{2}{*}{Herd.} & \multirow{2}{*}{DSA} & \multirow{2}{*}{DM} & \multirow{2}{*}{IDC-I} & \multirow{2}{*}{IDC} & Full \\
Class  &  &  &  &  &  &  & Dataset \\
\midrule                                   
10$\times$64$\times$64 (1\%) & 42.6 & 56.2 & 65.0 & 69.1 & 73.3 & \textbf{82.9} & \multirow{2}{*}{93.4}\\
20$\times$64$\times$64 (2\%) & 57.0 & 72.9 & 74.0 &	77.2 & 83.0 & \textbf{86.6} & \\
\bottomrule[1pt]
\end{tabular}
\end{adjustbox}
\label{tab:speech}
\vspace{-1em}
\end{table}

\subsection{Analysis}
\paragraph{Ablation Study.} \label{expsub:abl}
In this section, we perform an ablation study on our gradient matching techniques described in \Cref{analysis}. Specifically, we measure the isolated effect of 1) networks trained on real training data, 2) distance-based matching objective, and 3) stronger regularization on networks. \Cref{tab:abl} shows the ablation results of IDC-I on CIFAR-10 condensed with 50 images per class. From the table, we find that using MSE matching objective with networks trained on the synthetic dataset (\textit{Syn+MSE}) degenerates the performance significantly. However, when we use the MSE objective with networks trained on the real training dataset, the performance significantly increases compared to the baseline (\textit{DSA}), especially on ResNet-10. Furthermore, we find that strong regularization on networks brings additional performance improvements on both test models. The results demonstrate that the distance-based objective (\textit{MSE}) better distills training information than the similarity-based objective (\textit{Cos}) when using well-trained networks.

\begin{table}[t]
\caption{Ablation study of the proposed techniques (50 images per class on CIFAR-10). \textit{Syn} denotes condensing with networks trained on the synthetic dataset and \textit{Real} denotes condensing with networks trained on the real dataset. \textit{Cos} denotes cosine-similarity matching objective, \textit{MSE} denotes mean-square-error matching objective, and \textit{Reg.} denotes our proposed stronger regularization.}
\vspace{0.5em}
\centering
\begin{adjustbox}{max width=\columnwidth}
{\small
\begin{tabular}{l|ccccc}
\toprule[1pt]
\multirow{2}{*}{Test Model} & Syn+  & Syn+ & Real+ &  Real+ & Real+Reg.+ \\
 & Cos (DSA) & MSE & Cos &  MSE & MSE (Ours) \\
\midrule                                   
ConvNet-3 & 60.6 & 25.8 & 63.4 & 67.0 & \textbf{69.5}\\
ResNet-10 & 49.7 & 25.7 & 59.1 & 61.6 & \textbf{65.7}\\
\bottomrule[1pt]
\end{tabular}
}
\end{adjustbox}
\label{tab:abl}
\end{table}

\begin{table}[!t]
\vspace{-0.6em}
\caption{Test performance comparison of IDC and IDC-I with post-downsampling (IDC-I-post) on CIFAR-10. We denote the number of stored pixels in parenthesis.}
\vspace{0.5em}
\centering
\begin{adjustbox}{max width=\columnwidth}
\begin{tabular}{l|cc|c}
\toprule[1pt]
\multirow{2}{*}{Test Model} & IDC & IDC-I-post & IDC-I\\
          & (50$\times$32$\times$32) & (200$\times$16$\times$16) & (200$\times$32$\times$32)\\
\midrule                                   
ConvNet-3 & 74.5 & 68.8 & 76.6\\
ResNet-10 & 72.4 & 63.1 & 74.9\\
\bottomrule[1pt]
\end{tabular}
\end{adjustbox}
\label{tab:downsample}
\vspace{-0.8em}
\end{table}

\vspace{-0.5em}
\paragraph{Comparison to Post-Downsampling.}
One of the simple ways to save storage budget is to reduce the resolution of the synthesized data. In this subsection, we compare our end-to-end optimization framework to a post-downsampling approach which reduces the resolution of the optimized synthetic data and resizes the data to the original size at evaluation. \Cref{tab:downsample} shows IDC significantly outperforms IDC-I with post-downsampling under the same number of stored pixels, even approaching the performance of IDC-I without downsampling which stores 4 times more pixels. This result verifies the effectiveness of the end-to-end approach considering the formation function during the optimization process, \textit{i.e.}, finding the optimal condensed data given a fixed formation function.\looseness=-1

\paragraph{On Multi-Formation Factor.} We further study the effect of multi-formation factor (\textit{i.e.}, upsampling factor). \Cref{tab:resol} summarizes the test accuracy of condensed data with different multi-formation factors on various data scales. Note, the higher multi-formation factor results in a larger number of synthetic data but each with a lower resolution. \Cref{tab:resol} shows that datasets have different optimal multi-formation factors; 2 is optimal for CIFAR-10 and 3-4 are optimal for ImageNet. These results mean that there is a smaller room for trading off resolution in the case of CIFAR-10 than ImageNet where the input size is much larger. 

\begin{table}[t]
\vspace{-0.6em}
\caption{Condensation performance over various multi-formation factors on CIFAR-10 and ImageNet-10.}
\vspace{0.5em}
\centering
{\small
\begin{tabular}{ll|cccc}
\toprule[1pt]
Dataset  & Test & \multicolumn{4}{c}{Multi-Formation Factor} \\
(Pixel/Class) & Model  & 1 & 2 & 3 & 4 \\
\midrule                                   
CIFAR-10 & ConvNet-3 & 69.5 & \textbf{74.5} & 68.9 & 62.0 \\
(50$\times$32$\times$32) & ResNet-10 & 65.7 & \textbf{72.4} & 62.9 & 59.1 \\
\midrule                                   
ImageNet-10 & ResNetAP-10 & 65.5 & 73.3 & 76.6 & \textbf{77.5}  \\
(20$\times$224$\times$224) & ResNet-18 & 66.0 & 70.8 & \textbf{75.7} & 75.2 \\
\bottomrule[1pt]
\end{tabular}
}
\label{tab:resol}
\end{table}

\subsection{Application: Continual Learning}
Recent continual learning approaches include the process of constructing a small representative subset of data that has been seen so far and training it with newly observed data \citep{icarl,rainbowmemory}. This implies that the quality of the data subset is bound to affect the continual learning performance. In this section, we utilize the condensed data as exemplars for the previously seen classes or tasks and evaluate its effectiveness under the two types of continual learning settings: class incremental and task incremental \citep{dm,dsa}. Due to lack of space, we describe the detailed settings and results of task incremental setting in \Cref{appendix:cl}.

We follow the class incremental setting from \citet{dm}, where the CIFAR-100 dataset is given across 5 steps with a memory budget of 20 images per class. This setting trains a model continuously and purely on the latest data memory at each stage \citep{gdumb}. We synthesize the exemplars by only using the data samples of currently available classes at each stage with ConvNet-3. We evaluate the condensed data on two types of networks, ConvNet-3 and ResNet-10, and compare our condensation methods with Herding, DSA, and DM.

\Cref{fig:cl_gdumb} shows that IDC-I and IDC are superior to other baselines, both in ConvNet-3 and ResNet-10. Particularly, our multi-formation approach considerably increases the performance by over $10\%$p on average. In addition, from the results on ResNet-10, we find that DSA and DM do not maintain their performance under the network transfer, whereas our condensation methods outperform the baselines regardless of the networks types. That is, it is possible to efficiently condense data with small networks (ConvNet-3) and use the data on deeper networks when using our methods.\looseness=-1

\section{Related Work}
\label{relwork}
One of the classic approaches to establishing a compact representative subset of a huge dataset is coreset selection \citep{phillips2016coresets,forgetting}. Rather than selecting a subset,
\citet{hypergradient} originally proposed synthesizing a training dataset by optimizing the training performance. Following the work, \citet{such2020generative} introduce generative modeling for the synthetic dataset. However, these works do not consider storage efficiency. The seminal work by \citet{dd} studies synthesizing small training data with a limited storage budget. Building on this work, \citet{sucholutsky2021soft} attempt to co-optimize soft labels as well as the data, but they suffer from overfitting. 
Subsequently, \citet{kip} formulate the problem as kernel ridge regression and optimize the data based on neural tangent kernel. However, this approach requires hundreds of GPUs for condensation.  
\citet{dc} propose a scalable algorithm by casting the original bi-level optimization as a simpler matching problem. Following the work, \citet{dsa} exploit siamese augmentation to improve performance, and \citet{dm} suggest feature matching to accelerate optimization. Concurrently, \citet{trajectory} proposes to optimize the condensed data by matching training trajectories on the networks trained on real data.

\vspace{-1em}
\paragraph{Discussion on Dataset Structure} 
In this work, we constrain the condensation optimization variables (\textit{i.e.}, $\mathcal{S}$) to have the same shape as the original training data. This enables us to design an intuitive and efficient formation function that has negligible computation and storage overhead. However, if we deviate from pursuing the same shape, there exist a variety of considerable condensed data structures. For example, we can parameterize a dataset as dictionary phrase coding or neural network generator \cite{dictionary_coding,gan}. Nonetheless, it is not trivial to tailor these approaches for efficient data condensation. That is, it may require more storage or expensive computation costs for synthesis. For example, \citet{siren} use multi-layer neural networks that require much more storage than a single image to completely reconstruct a single image.\looseness=-1

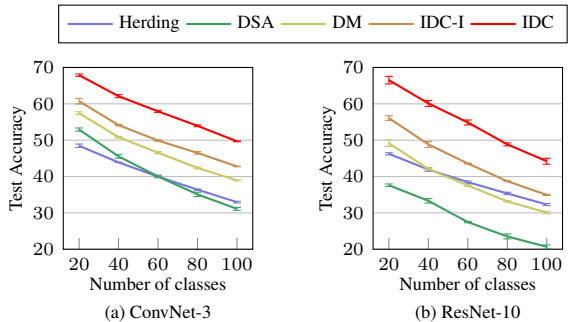
\begin{figure}[t]
    \begin{tikzpicture}[define rgb/.code={\definecolor{mycolor}{RGB}{#1}},
                    rgb color/.style={define rgb={#1},mycolor}]

\definecolor{gr}{RGB}{60,160,100}
\definecolor{or}{RGB}{200,140,80}
\definecolor{bl}{RGB}{120,120,220}
\definecolor{yl}{RGB}{200,200,100}
\definecolor{pp}{RGB}{200,150,240}

\begin{groupplot}[
        group style={columns=2, horizontal sep=1.6cm, 
        vertical sep=0.0cm},
        ]

\nextgroupplot[
            width=4.1cm,
            height=4.0cm,
            every axis plot/.append style={thick},
            ymajorgrids={true},
            scaled ticks = false,
            ytick style={draw=none},
            xlabel near ticks,
            ylabel near ticks,
            tick pos=left,
            tick label style={font=\scriptsize},
            xlabel shift=-0.15cm,         
            ylabel shift=-0.15cm,
            label style={font=\scriptsize},
            xlabel style={align=center},
            xlabel={Number of classes},
            ylabel={Test Accuracy},
            title style={at={(0.5,0)}, yshift=-1.3cm, font=\scriptsize},
            title = (a) ConvNet-3,
            xtick={20, 40, 60, 80, 100},
            ytick={20, 30, 40, 50, 60, 70},
            ymin=20,
            ymax=70,
            legend to name=grouplegend2,
    		legend style={legend columns=5, font=\scriptsize},
            ]
\coordinate (c1) at (rel axis cs:0,1);

\addplot[bl, mark size=1.5pt] plot [error bars/.cd, y dir=both, y explicit] table [y=herding, y error=herding-err, col sep=comma]{data/cl_gdumb_conv.csv};\addlegendentry{Herding}

\addplot[gr, mark size=1.5pt] plot [error bars/.cd, y dir=both, y explicit] table [y=dsa, y error=dsa-err, col sep=comma]{data/cl_gdumb_conv.csv};\addlegendentry{DSA}

\addplot[yl, mark size=1.5pt] plot [error bars/.cd, y dir=both, y explicit] table [y=dm, y error=dm-err, col sep=comma]{data/cl_gdumb_conv.csv};\addlegendentry{DM}

\addplot[or, mark size=1.5pt] plot [error bars/.cd, y dir=both, y explicit] table [y=mdci, y error=mdci-err, col sep=comma]{data/cl_gdumb_conv.csv};\addlegendentry{IDC-I}

\addplot[red, mark size=1.5pt] plot [error bars/.cd, y dir=both, y explicit] table [y=mdc, y error=mdc-err, col sep=comma]{data/cl_gdumb_conv.csv};\addlegendentry{IDC}

\nextgroupplot[
            width=4.1cm,
            height=4.0cm,
            every axis plot/.append style={thick},
            ymajorgrids={true},
            scaled ticks = false,
            ytick style={draw=none},
            xlabel near ticks,
            ylabel near ticks,
            tick pos=left,
            tick label style={font=\scriptsize},
            xlabel shift=-0.15cm,         
            ylabel shift=-0.15cm,
            label style={font=\scriptsize},
            xlabel style={align=center},
            xlabel={Number of classes},
            ylabel={Test Accuracy},
            title style={at={(0.5,0)}, yshift=-1.3cm, font=\scriptsize},
            title = (b) ResNet-10,
            xtick={20, 40, 60, 80, 100},
            ytick={20, 30, 40, 50, 60, 70},
            ymin=20,
            ymax=70,
            ]
\coordinate (c2) at (rel axis cs:1,1);

\addplot[bl, mark size=1.5pt] plot [error bars/.cd, y dir=both, y explicit] table [y=herding, y error=herding-err, col sep=comma]{data/cl_gdumb_resnet.csv};

\addplot[gr, mark size=1.5pt] plot [error bars/.cd, y dir=both, y explicit] table [y=dsa, y error=dsa-err, col sep=comma]{data/cl_gdumb_resnet.csv};

\addplot[yl, mark size=1.5pt] plot [error bars/.cd, y dir=both, y explicit] table [y=dm, y error=dm-err, col sep=comma]{data/cl_gdumb_resnet.csv};

\addplot[or, mark size=1.5pt] plot [error bars/.cd, y dir=both, y explicit] table [y=mdci, y error=mdci-err, col sep=comma]{data/cl_gdumb_resnet.csv};

\addplot[red, mark size=1.5pt] plot [error bars/.cd, y dir=both, y explicit] table [y=mdc, y error=mdc-err, col sep=comma]{data/cl_gdumb_resnet.csv};

\end{groupplot}

\coordinate (c3) at ($(c1)!.5!(c2)$);
\node[above] at (c3 |- current bounding box.north) {\pgfplotslegendfromname{grouplegend2}};
\end{tikzpicture}
    \vspace{-1.1em}
    \caption{Top-1 test accuracy of continual learning with condensed exemplars on CIFAR-100.}
    \label{fig:cl_gdumb}
    \vspace{-1.5em}
\end{figure}

\section{Conclusion}
\label{conclusion}
In this study, we address difficulties in optimization and propose a novel framework and techniques for dataset condensation. We propose a multi-formation process that defines enlarged and regularized data space for synthetic data optimization. We further analyze the shortcomings of the existing gradient matching algorithm and provide effective solutions. Our algorithm optimizes condensed data that achieve state-of-the-art performance in various experimental settings including speech domain and continual learning. 

\section*{Acknowledgement}
We are grateful to Junsuk Choe for helpful discussions. This work was supported by SNU-NAVER Hyperscale AI Center, Institute of Information \& Communications Technology Planning \& Evaluation (IITP) grant funded by the Korea government (MSIT) (No. 2020-0-00882, (SW STAR LAB) Development of deployable learning intelligence via self-sustainable and trustworthy machine learning and No. 2022-0-00480, Development of Training and Inference Methods for Goal-Oriented Artificial Intelligence Agents), and Basic Science Research Program through the National Research Foundation of Korea (NRF) (2020R1A2B5B03095585). Hyun Oh Song is the corresponding author.

\bibliography{main}

\begin{thebibliography}{42}
\providecommand{\natexlab}[1]{#1}
\providecommand{\url}[1]{\texttt{#1}}
\expandafter\ifx\csname urlstyle\endcsname\relax
  \providecommand{\doi}[1]{doi: #1}\else
  \providecommand{\doi}{doi: \begingroup \urlstyle{rm}\Url}\fi

\bibitem[Bang et~al.(2021)Bang, Kim, Yoo, Ha, and Choi]{rainbowmemory}
Bang, J., Kim, H., Yoo, Y., Ha, J.-W., and Choi, J.
\newblock Rainbow memory: Continual learning with a memory of diverse samples.
\newblock In \emph{CVPR}, 2021.

\bibitem[Bishop(2006)]{bishop}
Bishop, C.~M.
\newblock \emph{Pattern recognition and machine learning}.
\newblock springer, 2006.

\bibitem[Brown et~al.(2020)Brown, Mann, Ryder, Subbiah, Kaplan, Dhariwal,
  Neelakantan, Shyam, Sastry, Askell, et~al.]{gpt}
Brown, T.~B., Mann, B., Ryder, N., Subbiah, M., Kaplan, J., Dhariwal, P.,
  Neelakantan, A., Shyam, P., Sastry, G., Askell, A., et~al.
\newblock Language models are few-shot learners.
\newblock In \emph{NeurIPS}, 2020.

\bibitem[Cazenavette et~al.(2022)Cazenavette, Wang, Torralba, Efros, and
  Zhu]{trajectory}
Cazenavette, G., Wang, T., Torralba, A., Efros, A.~A., and Zhu, J.-Y.
\newblock Dataset distillation by matching training trajectories.
\newblock \emph{arXiv preprint arXiv:2203.11932}, 2022.

\bibitem[Cubuk et~al.(2019)Cubuk, Zoph, Mane, Vasudevan, and Le]{autoaugment}
Cubuk, E.~D., Zoph, B., Mane, D., Vasudevan, V., and Le, Q.~V.
\newblock Autoaugment: Learning augmentation strategies from data.
\newblock In \emph{CVPR}, 2019.

\bibitem[Deng et~al.(2009)Deng, Dong, Socher, Li, Li, and Fei-Fei]{imagenet}
Deng, J., Dong, W., Socher, R., Li, L.-J., Li, K., and Fei-Fei, L.
\newblock Imagenet: A large-scale hierarchical image database.
\newblock In \emph{CVPR}, 2009.

\bibitem[Dong et~al.(2016)Dong, Loy, and Tang]{fsrcnn}
Dong, C., Loy, C.~C., and Tang, X.
\newblock Accelerating the super-resolution convolutional neural network.
\newblock In \emph{ECCV}, pp.\  391--407. Springer, 2016.

\bibitem[Dosovitskiy et~al.(2021)Dosovitskiy, Beyer, Kolesnikov, Weissenborn,
  Zhai, Unterthiner, Dehghani, Minderer, Heigold, Gelly, et~al.]{vit}
Dosovitskiy, A., Beyer, L., Kolesnikov, A., Weissenborn, D., Zhai, X.,
  Unterthiner, T., Dehghani, M., Minderer, M., Heigold, G., Gelly, S., et~al.
\newblock An image is worth 16x16 words: Transformers for image recognition at
  scale.
\newblock In \emph{ICLR}, 2021.

\bibitem[Goodfellow et~al.(2014)Goodfellow, Pouget-Abadie, Mirza, Xu,
  Warde-Farley, Ozair, Courville, and Bengio]{gan}
Goodfellow, I., Pouget-Abadie, J., Mirza, M., Xu, B., Warde-Farley, D., Ozair,
  S., Courville, A., and Bengio, Y.
\newblock Generative adversarial nets.
\newblock In \emph{NeurIPS}, 2014.

\bibitem[He et~al.(2016)He, Zhang, Ren, and Sun]{resnet}
He, K., Zhang, X., Ren, S., and Sun, J.
\newblock Deep residual learning for image recognition.
\newblock In \emph{CVPR}, 2016.

\bibitem[Huang et~al.(2017)Huang, Liu, Van Der~Maaten, and
  Weinberger]{densenet}
Huang, G., Liu, Z., Van Der~Maaten, L., and Weinberger, K.~Q.
\newblock Densely connected convolutional networks.
\newblock In \emph{CVPR}, 2017.

\bibitem[Huang \& Mumford(1999)Huang and Mumford]{huang1999}
Huang, J. and Mumford, D.
\newblock Statistics of natural images and models.
\newblock In \emph{CVPR}, 1999.

\bibitem[Jastrzebski et~al.(2019)Jastrzebski, Kenton, Ballas, Fischer, Bengio,
  and Storkey]{hessian}
Jastrzebski, S., Kenton, Z., Ballas, N., Fischer, A., Bengio, Y., and Storkey,
  A.
\newblock On the relation between the sharpest directions of dnn loss and the
  sgd step length.
\newblock In \emph{ICLR}, 2019.

\bibitem[Kim et~al.(2020)Kim, Choo, and Song]{puzzle}
Kim, J.-H., Choo, W., and Song, H.~O.
\newblock Puzzle mix: Exploiting saliency and local statistics for optimal
  mixup.
\newblock In \emph{International Conference on Machine Learning (ICML)}, 2020.

\bibitem[Kim et~al.(2021)Kim, Choo, Jeong, and Song]{comix}
Kim, J.-H., Choo, W., Jeong, H., and Song, H.~O.
\newblock Co-mixup: Saliency guided joint mixup with supermodular diversity.
\newblock In \emph{ICLR}, 2021.

\bibitem[Krizhevsky et~al.(2009)Krizhevsky, Hinton, et~al.]{cifar}
Krizhevsky, A., Hinton, G., et~al.
\newblock Learning multiple layers of features from tiny images.
\newblock \emph{Citeseer}, 2009.

\bibitem[LeCun et~al.(2015)LeCun, Bengio, and Hinton]{deeplearning}
LeCun, Y., Bengio, Y., and Hinton, G.
\newblock Deep learning.
\newblock \emph{Nature}, 521\penalty0 (7553), 2015.

\bibitem[Li \& Hoiem(2017)Li and Hoiem]{lwf}
Li, Z. and Hoiem, D.
\newblock Learning without forgetting.
\newblock \emph{IEEE transactions on pattern analysis and machine
  intelligence}, 40\penalty0 (12), 2017.

\bibitem[Maclaurin et~al.(2015)Maclaurin, Duvenaud, and Adams]{hypergradient}
Maclaurin, D., Duvenaud, D., and Adams, R.
\newblock Gradient-based hyperparameter optimization through reversible
  learning.
\newblock In \emph{ICML}, 2015.

\bibitem[Mairal et~al.(2009)Mairal, Bach, Ponce, and Sapiro]{dictionary_coding}
Mairal, J., Bach, F., Ponce, J., and Sapiro, G.
\newblock Online dictionary learning for sparse coding.
\newblock In \emph{ICML}, 2009.

\bibitem[McLachlan \& Krishnan(2007)McLachlan and Krishnan]{em}
McLachlan, G.~J. and Krishnan, T.
\newblock \emph{The EM algorithm and extensions}, volume 382.
\newblock John Wiley \& Sons, 2007.

\bibitem[Nguyen et~al.(2021)Nguyen, Novak, Xiao, and Lee]{kip}
Nguyen, T., Novak, R., Xiao, L., and Lee, J.
\newblock Dataset distillation with infinitely wide convolutional networks.
\newblock In \emph{NeurIPS}, 2021.

\bibitem[Parisi et~al.(2019)Parisi, Kemker, Part, Kanan, and
  Wermter]{cl_review}
Parisi, G.~I., Kemker, R., Part, J.~L., Kanan, C., and Wermter, S.
\newblock Continual lifelong learning with neural networks: A review.
\newblock \emph{Neural Networks}, 113, 2019.

\bibitem[Patterson et~al.(2021)Patterson, Gonzalez, Le, Liang, Munguia,
  Rothchild, So, Texier, and Dean]{carbon}
Patterson, D., Gonzalez, J., Le, Q., Liang, C., Munguia, L.-M., Rothchild, D.,
  So, D., Texier, M., and Dean, J.
\newblock Carbon emissions and large neural network training.
\newblock \emph{arXiv preprint arXiv:2104.10350}, 2021.

\bibitem[Phillips(2016)]{phillips2016coresets}
Phillips, J.~M.
\newblock Coresets and sketches.
\newblock \emph{arXiv preprint arXiv:1601.00617}, 2016.

\bibitem[Prabhu et~al.(2020)Prabhu, Torr, and Dokania]{gdumb}
Prabhu, A., Torr, P.~H., and Dokania, P.~K.
\newblock Gdumb: A simple approach that questions our progress in continual
  learning.
\newblock In \emph{ECCV}, 2020.

\bibitem[Rebuffi et~al.(2017)Rebuffi, Kolesnikov, Sperl, and Lampert]{icarl}
Rebuffi, S.-A., Kolesnikov, A., Sperl, G., and Lampert, C.~H.
\newblock icarl: Incremental classifier and representation learning.
\newblock In \emph{CVPR}, 2017.

\bibitem[Sitzmann et~al.(2020)Sitzmann, Martel, Bergman, Lindell, and
  Wetzstein]{siren}
Sitzmann, V., Martel, J.~N., Bergman, A.~W., Lindell, D.~B., and Wetzstein, G.
\newblock Implicit neural representations with periodic activation functions.
\newblock In \emph{NeurIPS}, 2020.

\bibitem[Such et~al.(2020)Such, Rawal, Lehman, Stanley, and
  Clune]{such2020generative}
Such, F.~P., Rawal, A., Lehman, J., Stanley, K., and Clune, J.
\newblock Generative teaching networks: Accelerating neural architecture search
  by learning to generate synthetic training data.
\newblock In \emph{ICML}, 2020.

\bibitem[Sucholutsky \& Schonlau(2021)Sucholutsky and
  Schonlau]{sucholutsky2021soft}
Sucholutsky, I. and Schonlau, M.
\newblock Soft-label dataset distillation and text dataset distillation.
\newblock In \emph{2021 International Joint Conference on Neural Networks
  (IJCNN)}, 2021.

\bibitem[Tan \& Le(2019)Tan and Le]{efficientnet}
Tan, M. and Le, Q.
\newblock Efficientnet: Rethinking model scaling for convolutional neural
  networks.
\newblock In \emph{ICML}, 2019.

\bibitem[Tian et~al.(2020)Tian, Krishnan, and Isola]{imagenetsubset}
Tian, Y., Krishnan, D., and Isola, P.
\newblock Contrastive multiview coding.
\newblock In \emph{ECCV}, 2020.

\bibitem[Toneva et~al.(2019)Toneva, Sordoni, Combes, Trischler, Bengio, and
  Gordon]{forgetting}
Toneva, M., Sordoni, A., Combes, R. T.~d., Trischler, A., Bengio, Y., and
  Gordon, G.~J.
\newblock An empirical study of example forgetting during deep neural network
  learning.
\newblock In \emph{ICLR}, 2019.

\bibitem[Wang et~al.(2018)Wang, Zhu, Torralba, and Efros]{dd}
Wang, T., Zhu, J.-Y., Torralba, A., and Efros, A.~A.
\newblock Dataset distillation.
\newblock \emph{arXiv preprint arXiv:1811.10959}, 2018.

\bibitem[Warden(2018)]{google_speech}
Warden, P.
\newblock Speech commands: A dataset for limited-vocabulary speech recognition.
\newblock \emph{arXiv preprint arXiv:1804.03209}, 2018.

\bibitem[Welling(2009)]{herding}
Welling, M.
\newblock Herding dynamical weights to learn.
\newblock In \emph{ICML}, 2009.

\bibitem[Yun et~al.(2019)Yun, Han, Oh, Chun, Choe, and Yoo]{cutmix}
Yun, S., Han, D., Oh, S.~J., Chun, S., Choe, J., and Yoo, Y.
\newblock Cutmix: Regularization strategy to train strong classifiers with
  localizable features.
\newblock In \emph{ICCV}, 2019.

\bibitem[Zhang et~al.(2017)Zhang, Pezeshki, Brakel, Zhang, Bengio, and
  Courville]{speech_recog}
Zhang, Y., Pezeshki, M., Brakel, P., Zhang, S., Bengio, C. L.~Y., and
  Courville, A.
\newblock Towards end-to-end speech recognition with deep convolutional neural
  networks.
\newblock \emph{arXiv preprint arXiv:1701.02720}, 2017.

\bibitem[Zhao \& Bilen(2021{\natexlab{a}})Zhao and Bilen]{dm}
Zhao, B. and Bilen, H.
\newblock Dataset condensation with distribution matching.
\newblock \emph{arXiv preprint arXiv:2110.04181}, 2021{\natexlab{a}}.

\bibitem[Zhao \& Bilen(2021{\natexlab{b}})Zhao and Bilen]{dsa}
Zhao, B. and Bilen, H.
\newblock Dataset condensation with differentiable siamese augmentation.
\newblock In \emph{ICML}, 2021{\natexlab{b}}.

\bibitem[Zhao et~al.(2021)Zhao, Mopuri, and Bilen]{dc}
Zhao, B., Mopuri, K.~R., and Bilen, H.
\newblock Dataset condensation with gradient matching.
\newblock In \emph{ICLR}, 2021.

\bibitem[Zoph \& Le(2017)Zoph and Le]{nas}
Zoph, B. and Le, Q.~V.
\newblock Neural architecture search with reinforcement learning.
\newblock In \emph{ICLR}, 2017.

\end{thebibliography}
\bibliographystyle{icml2022}

\clearpage

\appendix
\section{Theoretical Analysis} 
\subsection{Proofs}\label{proof}
\setcounter{prop}{0}
\setcounter{definition}{0}

\begin{definition}
A function {\small$D:\mathcal{D}\times\mathcal{D}\to [0,\infty)$} is a dataset distance measure, if it satisfies the followings:
{\small$\forall X, X'\in \mathcal{D}$} where {\small$X=\{d_i\}_{i=1}^k$}, {\small$\forall i \in[1,\ldots,k]$},
\vspace{-0.6em}
\begin{enumerate}\itemsep=0pt
    \item {\small$D(X,X)=0$} and {\small$D(X,X')=D(X',X)$}.
    \item {\small$\forall d\in \mathbb{R}^m\ s.t.\ d$} is closer to {\small$X'$} than {\small$d_{i},\ D(X\setminus\{d_{i}\}\cup{\{d\}}, X') \le D(X, X')$}.
    \item {\small$D(X, X'\cup\{d_{i}\})\le D(X, X')$}.
\end{enumerate}
\label{def:measure}
\end{definition}
Note, we say data $d$ is closer to dataset {\small$X=\{d_i\}_{i=1}^{k}$} than $d'$, if {\small$\forall i\in[1,\ldots,k],\ \|d-d_i\|\le\|d'-d_i\|$}.

\begin{prop}
    If {\small$\mathcal{N}^{n'}\subseteq \mathcal{M}_f$}, then for any dataset distance measure $D$,
    
    \vspace{-2.2em}
    {\small
    \begin{align*}
        \min_{\mathcal{S}\in \mathbb{R}^{n\times m}} D(f(\mathcal{S}),\mathcal{T}) \le \min_{\mathcal{S}\in \mathbb{R}^{n\times m}} D(\mathcal{S},\mathcal{T}).
    \end{align*}
    }
\end{prop}

\begin{proof}
    For simplicity, we denote $[1,\ldots,n]$ as $[n]$.
    Let us denote $\mathcal{T}=\{t_i\}_{i=1}^{n_t}$ and $\mathcal{S}=\{s_j\}_{j=1}^{n}$, where $t_i \in \mathcal{N}\subset\mathbb{R}^m$ and $s_j\in\mathbb{R}^m$, $\forall i\in[n_t]$ and $\forall j\in[n]$. Under the assumption that $\mathcal{N}$ is a subspace of $\mathbb{R}^m$, there exists the projection of $s_j$ onto $\mathcal{N}$, $\bar{s}_j\in\mathcal{N}$. Because $t_i \in \mathcal{N}$ for $i=1,\ldots,n_t$, {\small$\|\bar{s}_j-t_i\|\le\|s_j-t_i\|$}, $\forall j\in[n]$ and $\forall i\in[n_t]$. This means
    the projection $\bar{s}_j$ is closer to $\mathcal{T}$ than $s_j$, $\forall j\in[n]$. Let us define a partially projected dataset {\small$\bar{\mathcal{S}}_{k}=\{\bar{s}_j\}_{j=1}^{k}\cup\{s_j\}_{j=k+1}^{n}$}. Then by the second axiom of \Cref{def:measure}, 

    \vspace{-1.2em}
    {\small
    \begin{align*}
        D(\bar{\mathcal{S}}_{n}, \mathcal{T})\le D(\bar{\mathcal{S}}_{n-1}, \mathcal{T})\le\ldots\le D(\mathcal{S}, \mathcal{T}).
    \end{align*}
    }
    
    \vspace{-0.5em}
    This result means that the optimum {\small$\mathcal{S}^*=\argmin D(\mathcal{S},\mathcal{T})$} satisfies {\small$\mathcal{S}^*\in\mathcal{N}^{n}$}. Note our multi-formation augments the number of data from $n$ to $n'$ where $n<n'$. Let us denote $k'=n'-n$ and {\small$S^*_{add}=\mathcal{S}^*\cup\{t_i\}_{i=1}^{k'}$}.
    By the third axiom of \Cref{def:measure}, 

    \vspace{-2em}
    {\small
    \begin{align*}
        D(S^*_{add}, \mathcal{T})\le D(\mathcal{S}^*, \mathcal{T}).
    \end{align*}
    }
    
    \vspace{-1.em}
    The elements of {\small$S^*_{add}$} lie in $\mathcal{N}$ and {\small$S^*_{add}\in\mathcal{N}^{n'}$}. From the assumption {\small$\mathcal{N}^{n'}\subseteq \mathcal{M}_f,\ \exists \mathcal{S}\in\mathbb{R}^{n\times m}$} \textit{s.t.} {\small$f(\mathcal{S})=S^*_{add}$}. Thus,
    
    \vspace{-2em}
    {\small
    \begin{align*}
        \min_{\mathcal{S}\in \mathbb{R}^{n\times m}} D(f(\mathcal{S}),\mathcal{T})&\le D(S^*_{add}, \mathcal{T})\\
        &\le D(\mathcal{S}^*, \mathcal{T})= \min_{\mathcal{S}\in \mathbb{R}^{n\times m}} D(\mathcal{S},\mathcal{T}).
    \end{align*}
    }
\end{proof}

\begin{prop}
Let $w_t\in \mathbb{R}^{K}$ and $h_t\in \mathbb{R}^{W}$ each denote the convolution weights and hidden features at the $t^\textrm{th}$ layer given the input data $x$. Then, for a loss function $\ell$, $\frac{d \ell(x)}{d w_t} = \sum_i a_{t,i} h_{t,i}$, where $h_{t,i}\in \mathbb{R}^{K}$ denotes the $i^\text{th}$ convolution patch of $h_t$ and $a_{t,i}=\frac{d \ell(x)}{d w_t^\intercal h_{t,i}}\in \mathbb{R}$.
\end{prop}

\begin{proof}
To clarify, we note that we are learning flipped kernel during training. Then the convolution output of the $t^\textrm{th}$ layer $o_t$ becomes $[w_t^\intercal h_{t,1},\ldots,w_t^\intercal h_{t,n_o}]$, where $n_o=W-K-1$ denotes the number of convolution patches given the convolution stride of 1. Then from the chain rule, 

\vspace{-1.5em}
{\small
\begin{align*}
    {\frac{d \ell(x)}{d w_t}} &= {\frac{d \ell(x)}{d o_t}} {\frac{d o_t}{d w_t}}\\
    &= \left[\frac{d \ell(x)}{d w_t^\intercal h_{t,1}},\ldots,\frac{d \ell(x)}{d w_t^\intercal h_{t,n_o}}\right] \left[h_{t,1},\ldots,h_{t,n_o} \right]^\intercal\\
    &=\sum_{i=1}^{n_0} \frac{d \ell(x)}{d w_t^\intercal h_{t,i}}h_{t,i}.
\end{align*}
}
\end{proof}

\subsection{Proposition 1 with Relaxed Assumption}\label{theory:relaxed}
In \Cref{prop:bound}, we assume {\small$\mathcal{N}^{n'}\subseteq \mathcal{M}_f$} that the synthetic-dataset space by $f$ is sufficiently large to contain all data points in $\mathcal{N}$. Relaxing the assumption, we consider when {\small$\mathcal{M}_f$} approximately covers {\small$\mathcal{N}^{n'}$}. With the following notion of $\epsilon$-cover, we describe the trade-off between the effects from the increase in the number of data and the decrease in representability of the synthetic datasets.

\begin{definition}
    Given a dataset distance measure $D$, {\small$\mathcal{M}_f$} is a $\epsilon$-cover of {\small$\mathcal{N}^{n'}$} on $D$ if {\small$\forall X'\in\mathcal{N}^{n'}$}, {\small$\exists S\in\mathbb{R}^{n\times m}$} \textit{s.t.} {\small$D(f(S),X')\le\epsilon$}.
\end{definition}

Here, we assume a dataset distance measure $D$ satisfies the triangular inequality.
From the proof above in Proposition 1, {\small$\exists S^*_{add}\in\mathcal{N}^{n'}$} \textit{s.t.}  {\small$D(S^*_{add}, \mathcal{T})\le\min_{\mathcal{S}\in \mathbb{R}^{n\times m}} D(\mathcal{S},\mathcal{T})$}. Let us denote the gain {\small$G=\min_{\mathcal{S}} D(\mathcal{S},\mathcal{T})-D(S^*_{add}, \mathcal{T})$}. If {\small$\mathcal{M}_f$} is a $\epsilon$-cover of {\small$\mathcal{N}^{n'}$} on $D$, then {\small$\exists \mathcal{S}\in\mathbb{R}^{n\times m}$} \textit{s.t.} 

\vspace{-1.5em}
{\small
\begin{align*}
   D(f(S),\mathcal{T})&\le D(S^*_{add},\mathcal{T})+D(f(S),S^*_{add})\\
   &\le D(S^*_{add},\mathcal{T})+\epsilon.
\end{align*}
}

\vspace{-1.em}
Note, we use the triangular inequality in the first inequality above and use the definition of $\epsilon$-cover in the second inequality. We can conclude that

\vspace{-1.5em}
{\small
\begin{align*}
    \min_{\mathcal{S}\in \mathbb{R}^{n\times m}} D(f(\mathcal{S}),\mathcal{T})&\le D(S^*_{add},\mathcal{T})+\epsilon\\
   &= \min_{\mathcal{S}\in \mathbb{R}^{n\times m}} D(\mathcal{S},\mathcal{T})-G+\epsilon.
\end{align*}
}

\vspace{-1.em}
To summarize, the optimization with multi-formation function $f$ can generate a synthetic dataset that has at least $G-\epsilon$ smaller distance to the original data $\mathcal{T}$ compared to when not using $f$. We can interpret $G$ as a possible gain by the increase in the number of data, \textit{i.e.}, from $n$ to $n'$, and $\epsilon$ as the representability loss in parameterization by $f$. This formulates a new research problem on data condensation: parameterization of a larger synthetic dataset that sufficiently satisfies the data regularity conditions.

\section{Datasets}
\paragraph{ImageNet-subset.} \label{data:imagenet}
Many recent works in machine learning evaluate their proposed algorithms using subclass samples from IamgeNet to validate the algorithms on large-scale data with a reasonable amount of computation resources \citep{icarl,imagenetsubset}. In our ImageNet experiment, we borrow a subclass list containing 100 classes from \citet{imagenetsubset}\footnote{\url{https://github.com/HobbitLong/CMC}}. We use the first 10 classes from the list in our ImageNet-10 experiments. We performed the experiments after preprocessing the images to a fixed size of $224\times224$ using resize and center crop functions.

\vspace{-0.5em}
\paragraph{Mini Speech Commands.} \label{data:speech}
This dataset contains one-second audio clips of 8 command classes, which is a subset of the original Speech Commands dataset \citep{google_speech}. The dataset consists of 1,000 samples for each class. We split the dataset randomly and use 7 of 8 as training data and 1 of 8 as test data. We downloaded the Mini Speech Commands from the official TensorFlow page\footnote{\url{https://www.tensorflow.org/tutorials/audio/simple_audio}}. Following the guideline provided on the official page, we load the audio clips with 16,000 sampling rates and process the waveform data with Short-time Fourier transform to obtain the magnitude spectrogram of size $128\times125$. Then, we apply zero-padding and perform downsampling to reduce the input size to $64\times64$. Finally, we use log-scale magnitude spectrograms for experiments.

\section{Implementation Details} 
\subsection{Ours} \label{imp_our}
In all of the experiments, we fix the number of inner iterations $M=100$ (\Cref{alg:matching}).
For CIFAR-10, we use data learning rate $\lambda=0.005$, network learning rate $\eta=0.01$, and the MSE objective. For other datasets, we use $L^1$ matching objective. For ImageNet-10, we use data learning rate $\lambda=0.003$  and network learning rate $\eta=0.01$. For ImageNet-100, we use data learning rate $\lambda=0.001$ and network learning rate $\eta=0.1$. For the speech dataset, we use data learning rate $\lambda=0.003$ and network learning rate $\eta=0.0003$. Rather than a single random augmentation strategy by DSA, we use a sequence of color transform, crop, and cutout for gradient matching. We train networks that are used for the matching with a sequence of color transform, crop, and CutMix \citep{cutmix}. Following \citet{dsa}, we initialize the synthetic data as random real training data samples, which makes the optimization faster compared to the random noise initialization.

We follow the evaluation setting by DSA in the case of CIFAR-10 \citep{dsa}. We train neural networks on the condensed data for 1,000 epochs with a 0.01 learning rate. We use the DSA augmentation and CutMix. Note, we apply CutMix for other baselines unless it degrades the performance. In the case of ImageNet, we train networks on the condensed data by using random resize-crop and CutMix. We use 0.01 learning rate and train models until convergence: 2,000 epochs for 10 image/class and 1,500 epochs for 20 image/class. We use an identical evaluation strategy for all cases in \Cref{tab:imgnet}.

\subsection{Baselines} \label{imp_base}
In the case of DSA \citep{dsa}, we download the author-released condensed dataset and evaluate the dataset\footnote{\url{https://github.com/VICO-UoE/DatasetCondensation}}. We train neural networks by following the training strategy from the official Github codes. We find that evaluation with CutMix degrades the performance of DSA, and report better results without CutMix in \Cref{tab:classification}. For all of the other baselines, we use an identical evaluation strategy to ours.

In the case of DM \citep{dm}, the codes are not released. Following the paper, we implemented the algorithm and tuned learning rates. As a result, we obtain superior performance than the reported values in \citet{dm}. Specifically, the original paper reports CIFAR-10 performance on ConvNet-3 of 63.0 whereas we obtain the performance of 65.6 (50 images per class). We report our improved results of DM in \Cref{tab:classification}.

\subsection{Continual Learning} \label{appendix:cl}
\paragraph{Class Incremental Setting.}
We reproduce the result based on the author released condensed dataset by \citet{dsa}. Instead of training the network from scratch as in previous works \citep{gdumb, dm}, we adopt distillation loss described in \citet{lwf} and train continuously by loading weights from the previous step and expanding the output dimension of the last fully-connected layer, which is a more realistic scenario in continual learning \citep{icarl}.
We train the model for 1000 epochs each stage using SGD with a learning rate of 0.01, decaying by a factor of 0.2 at epochs 600 and 800. We use 0.9 for momentum and 0.0005 for weight decay.

\vspace{-0.5em}
\paragraph{Task Incremental Setting.}
We follow the task incremental setting by \citet{dsa}, which consists of three digit-datasets (SVHN $\rightarrow$ MNIST $\rightarrow$ USPS). At each training stage, a new set of the corresponding data is provided, whereas the previously seen datasets are prohibited except for a few exemplars. We compare our methods with Herding and DSA, excluding DM where the data under this setting is not released. As shown in \Cref{fig:cl_digit}, we verify that our condensed data significantly outperform the baselines.

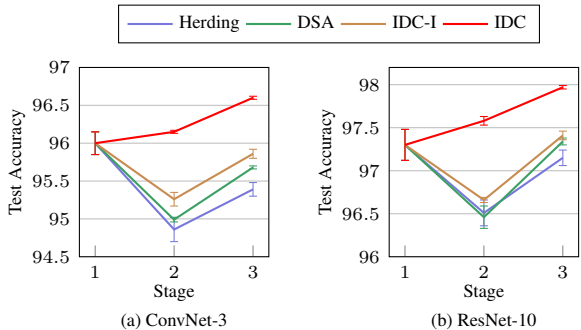
\begin{figure}[t]
    \begin{tikzpicture}[define rgb/.code={\definecolor{mycolor}{RGB}{#1}},
                    rgb color/.style={define rgb={#1},mycolor}]

\definecolor{gr}{RGB}{60,160,100}
\definecolor{or}{RGB}{200,140,80}
\definecolor{bl}{RGB}{120,120,220}

\begin{groupplot}[
        group style={columns=2, horizontal sep=1.6cm, 
        vertical sep=0.0cm},
        ]

\nextgroupplot[
            width=4.1cm,
            height=4.1cm,
            every axis plot/.append style={thick},
            ymajorgrids={true},
            scaled ticks = false,
            ytick style={draw=none},
            xlabel near ticks,
            ylabel near ticks,
            tick pos=left,
            tick label style={font=\scriptsize},
            xlabel shift=-0.15cm,         
            ylabel shift=-0.15cm,
            label style={font=\scriptsize},
            xlabel style={align=center},
            xlabel={Stage},
            ylabel={Test Accuracy},
            title style={at={(0.5,0)}, yshift=-1.3cm, font=\scriptsize},
            title = (a) ConvNet-3,
            xtick={1, 2, 3},
            ytick={94.5, 95, 95.5, 96, 96.5, 97},
            ymin=94.5,
            ymax=97,
            legend to name=grouplegend,
    		legend style={legend columns=4, font=\scriptsize},
            ]
\coordinate (c1) at (rel axis cs:0,1);

\addplot[bl, mark size=1.5pt] plot [error bars/.cd, y dir=both, y explicit] table [y=herding, y error=herding-err, col sep=comma]{data/cl_digit_conv.csv};\addlegendentry{Herding}

\addplot[gr, mark size=1.5pt] plot [error bars/.cd, y dir=both, y explicit] table [y=dsa, y error=dsa-err, col sep=comma]{data/cl_digit_conv.csv};\addlegendentry{DSA}

\addplot[or, mark size=1.5pt] plot [error bars/.cd, y dir=both, y explicit] table [y=mdci, y error=mdci-err, col sep=comma]{data/cl_digit_conv.csv};\addlegendentry{IDC-I}

\addplot[red, mark size=1.5pt] plot [error bars/.cd, y dir=both, y explicit] table [y=mdc, y error=mdc-err, col sep=comma]{data/cl_digit_conv.csv};\addlegendentry{IDC}

\nextgroupplot[
            width=4.1cm,
            height=4.1cm,
            every axis plot/.append style={thick},
            ymajorgrids={true},
            scaled ticks = false,
            ytick style={draw=none},
            xlabel near ticks,
            ylabel near ticks,
            tick pos=left,
            tick label style={font=\scriptsize},
            xlabel shift=-0.15cm,         
            ylabel shift=-0.15cm,
            label style={font=\scriptsize},
            xlabel style={align=center},
            xlabel={Stage},
            ylabel={Test Accuracy},
            title style={at={(0.5,0)}, yshift=-1.3cm, font=\scriptsize},
            title = (b) ResNet-10,
            xtick={1, 2, 3},
            ytick={96, 96.5, 97, 97.5, 98},
            ymin=96,
            ymax=98.2,
            ]
\coordinate (c2) at (rel axis cs:1,1);

\addplot[bl, mark size=1.5pt] plot [error bars/.cd, y dir=both, y explicit] table [y=herding, y error=herding-err, col sep=comma]{data/cl_digit_resnet.csv};

\addplot[gr, mark size=1.5pt] plot [error bars/.cd, y dir=both, y explicit] table [y=dsa, y error=dsa-err, col sep=comma]{data/cl_digit_resnet.csv};

\addplot[or, mark size=1.5pt] plot [error bars/.cd, y dir=both, y explicit] table [y=mdci, y error=mdci-err, col sep=comma]{data/cl_digit_resnet.csv};

\addplot[red, mark size=1.5pt] plot [error bars/.cd, y dir=both, y explicit] table [y=mdc, y error=mdc-err, col sep=comma]{data/cl_digit_resnet.csv};

\end{groupplot}

\coordinate (c3) at ($(c1)!.5!(c2)$);
\node[above] at (c3 |- current bounding box.north) {\pgfplotslegendfromname{grouplegend}};
\end{tikzpicture}
    \vspace{-1em}
    \caption{Continual learning performance with exemplars on digit datasets (SVHN-MNIST-USPS).}
    \label{fig:cl_digit}
\end{figure}

\section{Other Multi-Formation Functions}\label{appendix:formation}
In this section, we study other types of multi-formation functions that are worth considering under our framework. 

\vspace{-0.5em}
\paragraph{Multi-Scale Multi-Formation Function.}
The synthetic data by the uniform multi-formation function do not share data elements with each other (\Cref{fig:formation}). Here, we design a multi-scale formation function that increases the number of synthetic data by sharing condensed data elements across multiple synthetic data (right subfigure in \Cref{fig:formation}). \Cref{tab:multi-scale} compares the test accuracy to the default formation function on CIFAR-10. The table shows that the multi-scale approach outperforms the uniform formation function under the small storage budget where the uniform approach does not create sufficiently many synthetic data.

\vspace{-0.5em}
\paragraph{Learnable Multi-Formation Function.}
We further study the potential direction of exploiting learnable multi-formation function which can synthesize diverse representative images at the cost of additional computation overhead and storage.
In this experiment, we replace the upsampling by a  learnable function using Fast Super-Resolution Convolutional Neural Networks (FSRCNN) with a reduced number of parameters \citep{fsrcnn}.
\Cref{tab:learnable} summarizes the condensation performance of learnable multi-formation function with different factors on CIFAR-10.
While the extra learnable module does not improve performance with the formation factor of 2, it improves the performance with the factor of 3.
We conjecture that the upsampling can generate sufficiently informative synthetic data in the lower factor, but suffers from the lack of representability in the higher factor.
In such scenarios with the larger factor, the learnable multi-formation function shows promising results.

\begin{table}[t]
\vspace{-0.6em}
\caption{Test performance comparison of the uniform and multi-scale formation functions on CIFAR-10.}
\vspace{0.5em}
\centering
\small
\begin{tabular}{cl|cc}
\toprule[1pt]
Pixel/Class & Test Model& Uniform (default) & Multi-Scale\\
\midrule                                   
10$\times$32$\times$32 & ConvNet-3  & 67.5 & \textbf{69.2}\\
(0.2\%)                & ResNet-10  & 63.5 & \textbf{64.8}\\
\midrule                                   
50$\times$32$\times$32 & ConvNet-3 & \textbf{74.5} & 73.1\\
(1.0\%)                & ResNet-10 & \textbf{72.4} & 69.7 \\
\bottomrule[1pt]
\end{tabular}
\label{tab:multi-scale}
\end{table}

\begin{table}[!t]
    \vspace{-0.6em}
    \caption{Condensation performance comparison of learnable multi-formation functions to upsampling (10 images per class on CIFAR-10). \textit{CN} denotes ConvNet-3 and \textit{RN} denotes ResNet-10.} \vspace{0.5em}
    \centering
    \small
    \begin{tabular}{l|cccc}
    \toprule[1pt] 
    Test & \multicolumn{2}{c}{Factor 2} & \multicolumn{2}{c}{Factor 3} \\ Model & Upsample & FSRCNN & Upsample & FSRCNN \\  
    \midrule                                  
    CN  & 67.5 & 66.2 & 66.7 & \textbf{67.9}\\
    RN  & 63.5 & 62.0 & 60.6 & \textbf{64.4}\\
    \bottomrule[1pt]
    \end{tabular}
    \label{tab:learnable}
    \vspace{-1em}
\end{table}

\begin{table*}[t]
\vspace{-0.6em}
\centering
\caption{Top-1 test accuracy of ConvNet-3 trained on condensed datasets (average score with 3 evaluation repetitions).}
\vspace{0.3em}
{\scriptsize
\begin{tabular}{c|
c@{\hspace{1.06\tabcolsep}}c@{\hspace{1.28\tabcolsep}}c@{\hspace{1.33\tabcolsep}}c@{\hspace{1.31\tabcolsep}}c|
c@{\hspace{1.06\tabcolsep}}c@{\hspace{1.28\tabcolsep}}c@{\hspace{1.33\tabcolsep}}c@{\hspace{1.31\tabcolsep}}c|
c@{\hspace{1.06\tabcolsep}}c@{\hspace{1.28\tabcolsep}}c@{\hspace{1.33\tabcolsep}}c@{\hspace{1.31\tabcolsep}}c|
c@{\hspace{1.06\tabcolsep}}c@{\hspace{1.28\tabcolsep}}c@{\hspace{1.33\tabcolsep}}c@{\hspace{1.31\tabcolsep}}c}
\toprule[1pt]
Img/ & \multicolumn{5}{c}{SVHN} & \multicolumn{5}{c}{MNIST} & \multicolumn{5}{c}{FashionMNIST} &  \multicolumn{5}{c}{CIFAR-10}  \\
Cls & DSA$^\dagger$ & KIP$^\dagger$ & DM & IDC-I & IDC
 & DSA$^\dagger$ & KIP$^\dagger$ & DM & IDC-I & IDC
 & DSA$^\dagger$ & KIP$^\dagger$ & DM & IDC-I & IDC
 & DSA$^\dagger$ & KIP$^\dagger$ & DM & IDC-I & IDC \\
\midrule      
1 & 27.5 & 39.5 & 24.2 & 46.7 & \textbf{68.5} 
& 88.7 & 90.1 & 89.7 & 88.9 & \textbf{94.2} 
& 70.6 & 73.5 & 70.0 & 70.7 & \textbf{81.0}
& 28.8 & 38.6 & 28.9 & 36.7 & \textbf{50.6}\\
10 & 79.2 & 64.2 & 72.0 & 77.0 & \textbf{87.5}
& 97.8 & 97.5 & 97.5 & 98.0 & \textbf{98.4} 
& 84.6 & \textbf{86.8} & 84.8 & 85.3 & 86.0
& 52.1 & 49.2 & 53.8 & 58.3 & \textbf{67.5} \\
50 & 84.4 & 73.2 & 84.3 & 87.9 & \textbf{90.1} 
& \textbf{99.2} & 98.3 & 98.6 & 98.8 & \textbf{99.1}
& 88.7 & 88.0 & 88.6 & \textbf{89.1} & 86.2
& 60.6 & 56.7 & 65.6 & 69.5 & \textbf{74.5} \\
\bottomrule[1pt]
\end{tabular}
}
\label{tab:simple}
\vspace{-0.2em}
\end{table*}

\begin{table*}[t]
\vspace{-0.7em}
\centering
\caption{Ablation study on strong augmentation (S.A.), CIFAR-10 (ConvNet, 50 img/cls). We report bold values in \Cref{tab:classification,tab:computation}. Evaluation \textit{w/o} S.A. is identical to the method by DSA and DM.}
\vspace{0.3em}
{\small
\begin{tabular}{c|ccccc|cccc}
\toprule[1pt]
Evalution & Random & Herding	& DSA & KIP	& DM & IDC-I \footnotesize{\textit{w/o} S.A.} & IDC-I & IDC \footnotesize{\textit{w/o} S.A.} & IDC\\
\midrule      
\textit{w/o} S.A. & 54.7 & 57.5	& \textbf{60.6} & 55.8 & 63.0 & 67.4	& 68.6 & 72.8 & 73.5 \\
\textit{w/} S.A. & \textbf{56.5} & \textbf{59.8} & 59.5 & \textbf{57.9} & \textbf{65.6} & 67.0	& \textbf{69.5} & 74.3 & \textbf{74.5}\\
\bottomrule[1pt]
\end{tabular}
}
\label{tab:abl_sa}
\vspace{-0.5em}
\end{table*}

\section{Additional Experiments}

\subsection{Experimental Results on Other Datasets}
\label{appendix:other}
We evaluate our method on SVHN, MNIST, FashionMNIST, and CIFAR-10 including 1 img/cls setting and verify that our methods consistently outperform baselines. \Cref{tab:simple} shows multi-formation is much more effective at low compression rates (1 img/cls) and improves performance by up to 30\%p (on SVHN) compared to the best baseline. We also find that the effect of multi-formation is diminishing at (FashionMNIST, 50 img/cls) where IDC-I is the best. We conjecture that the representation loss by multi-formation at this point is greater than the gain by an increased number of data, which can be backed up by analysis in \Cref{theory:relaxed}.

\subsection{Isolated Effect of Strong Augmentation}\label{appendix:sa}
We conduct an ablation study investigating the effect of strong augmentation (S.A.), i.e., CutMix, in \Cref{tab:abl_sa}. We implement our algorithm without S.A. and evaluate all baselines under the two evaluation strategies: with or without S.A. The table shows that the gain by S.A. is only about 1\%p whereas the gain by multi-formation and algorithmic development is about 14\%p and 9\%p (by comparing IDC \textit{w/o} S.A. with DSA and DM). The result verifies that our algorithm does not mainly rely on augmentation.

\begin{figure}[t]
    \begin{tikzpicture}[define rgb/.code={\definecolor{mycolor}{RGB}{#1}},
                    rgb color/.style={define rgb={#1},mycolor}]

\definecolor{gr}{RGB}{60,160,100}
\definecolor{or}{RGB}{200,140,80}
\definecolor{bl}{RGB}{120,120,220}

\begin{groupplot}[
        group style={columns=1, horizontal sep=1.05cm, 
        vertical sep=0.0cm},
        ]

\nextgroupplot[
            width=6.6cm,
            height=4.3cm,
            every axis plot/.append style={thick},
            grid=major,
            scaled ticks = false,
            xlabel near ticks,
            ylabel near ticks,
            tick pos=left,
            tick label style={font=\scriptsize},
            xlabel shift=-0.1cm,         
            ylabel shift=-0.15cm,
            label style={font=\scriptsize},
            xlabel style={align=center},
            xlabel={Image per class},
            ylabel={Test Accuracy},
            xmin=1,
            xmax=500,
            symbolic x coords={1, 10, 50, 100, 200, 500},
            ytick={0,20,40,60,80,100},
            ymin=0,
            ymax=100,
            legend cell align=left,
            legend style={at={(1.05,1)},anchor=north west, nodes={scale=0.7}},
            ]

\addplot[bl, opacity=0.8, mark=*, mark size=0.6pt] table [y=random, col sep=comma]{data/ipc.csv};\addlegendentry{Random}

\addplot[gr, opacity=0.8, mark=*, mark size=0.6pt] table [y=herding, col sep=comma]{data/ipc.csv};\addlegendentry{Herding}

\addplot[or, opacity=0.8, mark=*, mark size=0.6pt] table [y=ours, col sep=comma]{data/ipc.csv};\addlegendentry{IDC-I}

\addplot[red, opacity=0.8, mark=*, mark size=0.6pt] table [y=ours2, col sep=comma]{data/ipc.csv};\addlegendentry{IDC}

\addplot[black, dotted] table [y=full, col sep=comma]{data/ipc.csv};\addlegendentry{Full}

\end{groupplot}
\end{tikzpicture}
    \vspace{-1.3em}
    \caption{Top-1 test accuracy of ConvNet-3 trained on condensed datasets with increasing data storage (CIFAR-10). 
    }
\label{fig:ipc}
\vspace{-1.3em}
\end{figure}
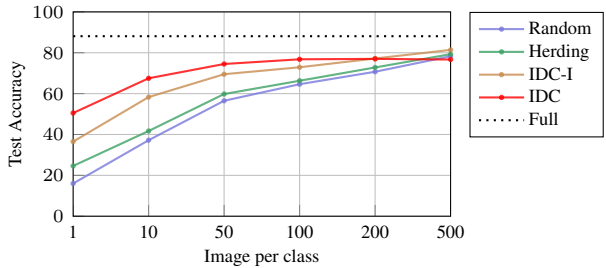

\subsection{Larger Data Storage} \label{appendix:ipc}
In this subsection, we measure the performance of condensation with larger storage budgets. \Cref{fig:ipc} shows the performance of condensed data with large storage budgets of up to 500 images per class on CIFAR-10. The figure shows that IDC outperforms other methods under the storage budgets of 200 images per class, however, IDC underperforms at 500 images per class. This result indicates that increasing the number of synthetic data via multi-formation shows diminishing returns when there are enough storage budgets to represent the original training data diversity (see \Cref{theory:relaxed} for theoretical analysis). Nonetheless, IDC-I outperforms baselines in all settings, demonstrating the effectiveness of our condensation algorithm with large storage budgets.\looseness=-1

\vspace{-0.4em}
\subsection{Network Architecture Effect on Condensation}
\label{appendix:architecture}
We analyze the effects of networks' architecture by comparing the performance of condensed data on CIFAR-10. In \Cref{tab:network}, we compare the performance of condensed data with different network architectures; simple convolution networks with various widths and depths, ResNet-10, and ResNetAP-10. Interestingly, simple ConvNets perform better than the deeper ResNet architectures on both test models. Furthermore, ConvNet-4 (\textit{CN-D}) has lower condensation performance than ConvNet-3 (\textit{CN}), although it has more convolution layers. The results indicate that a complex network is not always effective when compressing a large amount of learning information on a small storage capacity.

\vspace{-0.2em}
\subsection{Multi-Formation with Another Algorithm}
Our multi-formation strategy can be orthogonally applied to other condensation methods. To verify the generality of our strategy, we apply the multi-formation function on another condensation method, DM \citep{dm}, which uses a feature matching objective. \Cref{tab:dm_decoding} summarizes the test performance of condensed data on CIFAR-10. The table shows that our multi-formation framework consistently improves the performance of DM over various test models, demonstrating the general applicability of our framework.

\vspace{-0.2em}
\subsection{Dataset Condensation with ZCA}\label{appendix:zca}
We implement ZCA following the official KIP code \citep{kip} and test IDC on CIFAR-10 (\Cref{tab:zca}). We find that ZCA results in mild degradation in performance. We speculate that ZCA whitening, which removes pixel correlation, is not suitable for IDC's upsampling process.

\begin{table}[t]
\vspace{-0.5em}
\caption{Condensation network architecture comparison (10 img/cls on CIFAR-10). \textit{CN} denotes ConvNet-3, \textit{CN-W} denotes ConvNet-3 with twice more channels (256), \textit{CN-D} denotes ConvNet-4 (4 convolution layers), \textit{RN} denotes ResNet-10, and \textit{RN-AP} denotes ResNet-10 with average pooling instead of strided convolutions for downsampling. Note, we use instance normalization as in \citet{dc}.}
\vspace{0.5em}
\centering
{\small
\begin{tabular}{l|ccccc}
\toprule[1pt]
\multirow{2}{*}{Test Model} & \multicolumn{5}{c}{Condensation Network Architecture} \\
 &  CN & CN-W & CN-D & RN & RNAP\\
\midrule                                   
ConvNet-3 & 58.3 & \textbf{58.8} & 56.6 & 51.4 & 53.5\\
ResNet-10 & 50.2 & \textbf{50.5} & 48.7 & 47.5 & 48.8\\
\bottomrule[1pt]
\end{tabular}
}
\vspace{-0.5em}
\label{tab:network}
\end{table}

\begin{table}[t]
\caption{Test accuracy of feature matching objective by DM with our multi-formation strategy (DM+MF) on CIFAR-10.}
\vspace{0.5em}
\centering
\small
\begin{tabular}{cl|ccc}
\toprule[1pt]
Pixel/Class & Test Model& DM & DM+MF & IDC\\
\midrule                                   
50$\times$32$\times$32 & ConvNet-3 & 65.6 & 68.4 & 74.5\\
(1.0\%)                & ResNet-10 & 58.6 & 63.1 & 72.4\\
\bottomrule[1pt]
\end{tabular}
\label{tab:dm_decoding}
\vspace{-0.8em}
\end{table}

\begin{table}[!t]
\centering
\caption{Effects of ZCA whitening on IDC (CIFAR-10 with ConvNet-3). Here, \textit{S.A.} means strong augmentation.}
\vspace{0.5em}
{\small
\begin{tabular}{c|ccc}
\toprule[1pt]
Pixel/Class &  IDC \footnotesize{\textit{w/o} S.A.} + ZCA  & IDC + ZCA & IDC \\
\midrule      
10$\times$32$\times$32 & 66.6 & 66.7 &  \textbf{67.5} \\
50$\times$32$\times$32 & 72.0 & 72.5 &  \textbf{74.5} \\
\bottomrule[1pt]
\end{tabular}
}
\vspace{-0.5em}
\label{tab:zca}
\end{table}

\section{Visual Examples}
\label{appendix:visual}
We provide visual examples of IDC on CIFAR-10, ImageNet, MNIST, and SVHN in the following pages. 
In \Cref{fig:visual21,fig:visual22,fig:visual23,fig:visual24}, we compare our synthetic data samples to the real training data samples, which we used as initialization of the synthetic data. From the figure, we find that our synthesized data looks more class-representative.
We provide the full condensed data in \Cref{fig:visual12,fig:visual11,fig:visual13,fig:visual14}, under the storage budget of 10 images per class.

\begin{figure*}[!ht]
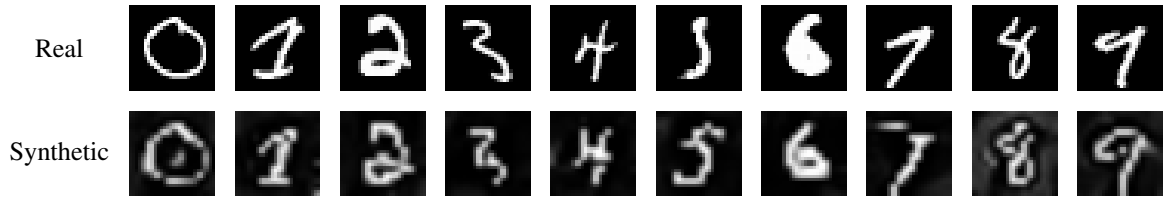

    \centering

\begin{tikzpicture}[decoration={brace}][scale=1,every node/.style={minimum size=1cm},on grid]

\foreach \j in {5,15,25,35,45,55,65,75,85,95}{
\foreach \i in {0, 1} {
    \ifnum \i=0
    \node (cifar) at (.14*\j,0) {\includegraphics[width=32px]{image/mnist/cd/data\j.png}};
    \else
    \node (cifar) at (.14*\j,1.4) {\includegraphics[width=32px]{image/mnist/real/data\j.png}};
    \fi
}
}
\node (real) at (-.8, 1.4) {Real};

\node (cd) at (-.8, -0.) {Synthetic};

\end{tikzpicture}

\vspace{-1em}
\caption{Comparison of real and synthetic images on MNIST.} 
\vspace{-1em}
\label{fig:visual22}
\end{figure*}

\begin{figure*}[!ht]
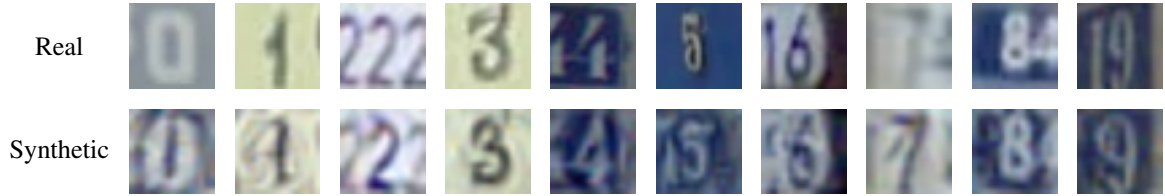

    \centering

\begin{tikzpicture}[decoration={brace}][scale=1,every node/.style={minimum size=1cm},on grid]

\foreach \j in {5,15,25,35,45,55,65,75,85,95}{
\foreach \i in {0, 1} {
    \ifnum \i=0
    \node (svhn) at (.14*\j,0) {\includegraphics[width=32px]{image/svhn/cd/data\j.png}};
    \else
    \node (svhn) at (.14*\j,1.4) {\includegraphics[width=32px]{image/svhn/real/data\j.png}};
    \fi
}
}
\node (real) at (-.8, 1.4) {Real};
\node (cd) at (-.8, -0.) {Synthetic};

\end{tikzpicture}

\vspace{-1em}
\caption{Comparison of real and synthetic images on SVHN.} 
\vspace{-1em}
\label{fig:visual21}
\end{figure*}

\begin{figure*}[!ht]
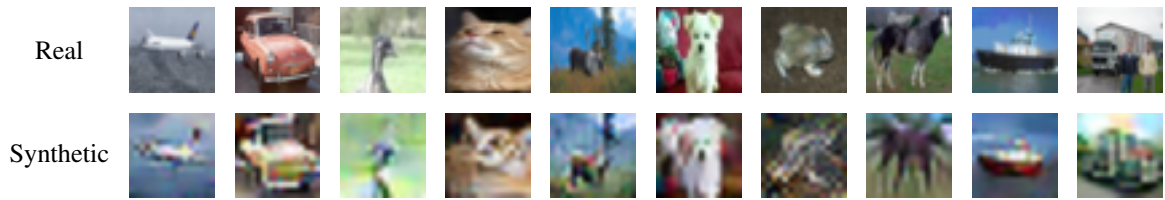

    \centering

\begin{tikzpicture}[decoration={brace}][scale=1,every node/.style={minimum size=1cm},on grid]

\foreach \j in {5,15,25,35,45,55,65,75,85,95}{
\foreach \i in {0, 1} {
    \ifnum \i=0
    \node (cifar) at (.14*\j,0) {\includegraphics[width=32px]{image/cifar/cd/data\j.png}};
    \else
    \node (cifar) at (.14*\j,1.4) {\includegraphics[width=32px]{image/cifar/real/data\j.png}};
    \fi
}
}
\node (real) at (-.8, 1.4) {Real};
\node (cd) at (-.8, 0.) {Synthetic};

\end{tikzpicture}

\vspace{-1em}
\caption{Comparison of real and synthetic images on CIFAR-10.} 
\vspace{-1em}
\label{fig:visual23}
\end{figure*}

\begin{figure*}[!ht]
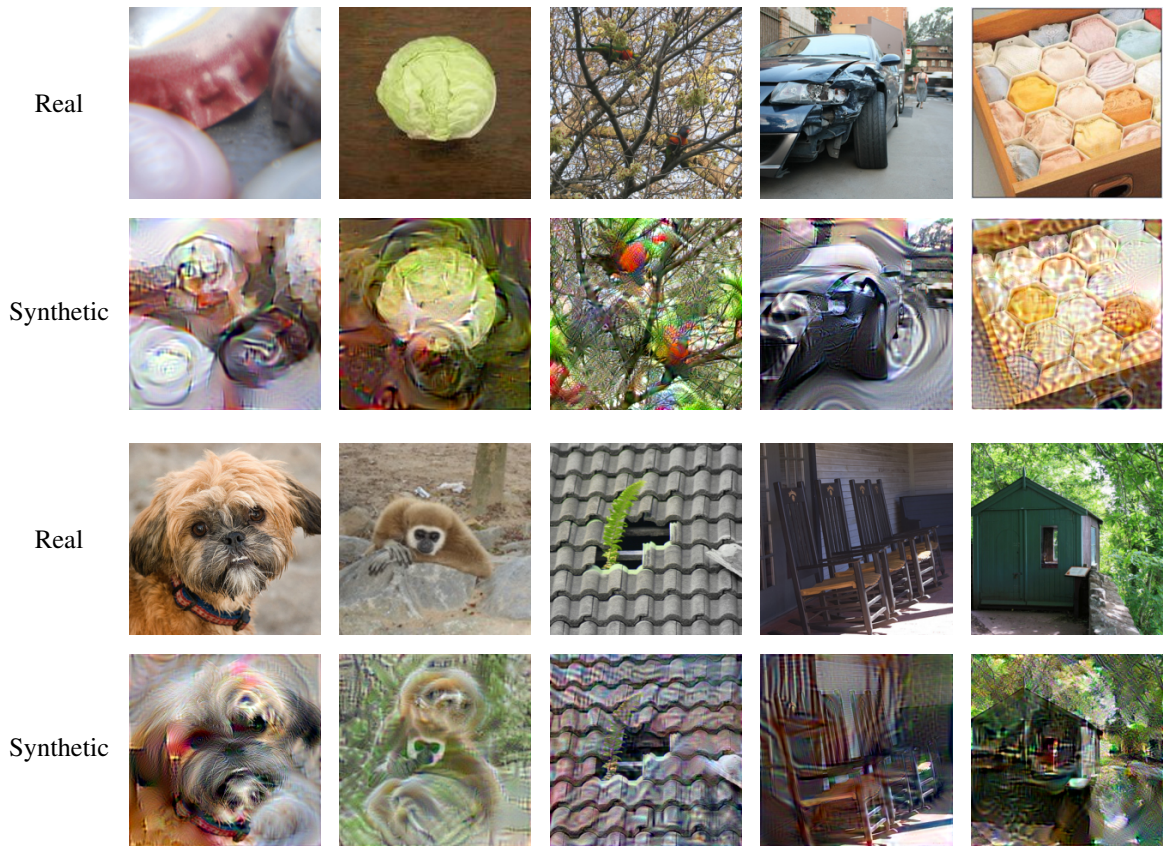

    \centering

\begin{tikzpicture}[decoration={brace}][scale=1,every node/.style={minimum size=1cm},on grid]

\foreach \j in {5,15,25,35,45}{
\foreach \i in {0, 1} {
    \ifnum \i=0
    \node (cifar) at (.28*\j,0) {\includegraphics[width=72px]{image/imagenet/cd/data\j.png}};
    \else
    \node (cifar) at (.28*\j,2.8) {\includegraphics[width=72px]{image/imagenet/real/data\j.png}};
    \fi
}
}

\foreach \j in {55,65,75,85,95}{
\foreach \i in {0, 1} {
    \ifnum \i=0
    \node (cifar) at (.28*\j-14,-5.8) {\includegraphics[width=72px]{image/imagenet/cd/data\j.png}};
    \else
    \node (cifar) at (.28*\j-14,-3.0) {\includegraphics[width=72px]{image/imagenet/real/data\j.png}};
    \fi
}
}

\node (real) at (-.8, 2.8) {Real};
\node (cd) at (-.8, 0.) {Synthetic};

\node (real) at (-.8, -3.0) {Real};
\node (cd) at (-.8, -5.8) {Synthetic};

\end{tikzpicture}

\vspace{-1em}
\caption{ImageNet: bottle cap, cabbage, lorikeet, car wheel, honeycomb, Shih-Tzu, gibbon, tile roof, rocking chair, and boat house.} 
\vspace{-1em}
\label{fig:visual24}
\end{figure*}
\begin{figure}[!ht]
    \centering

\begin{tikzpicture}[decoration={brace}][scale=1,every node/.style={minimum size=1cm},on grid]
    \node (mnist){\includegraphics[width=\columnwidth]{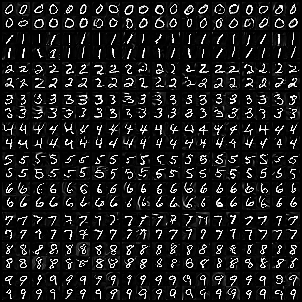}};
\end{tikzpicture}

\vspace{-1em}
\caption{Condensed images of MNIST dataset with IDC ($10\times28\times28$ pixels per class).} 
\vspace{-1em}
\label{fig:visual13}
\end{figure}

\begin{figure}[!ht]
    \centering

\begin{tikzpicture}[decoration={brace}][scale=1,every node/.style={minimum size=1cm},on grid]
    \node (cifar){\includegraphics[width=\columnwidth]{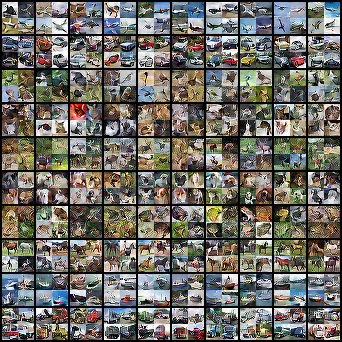}};
\end{tikzpicture}

\vspace{-1em}
\caption{Condensed images of CIFAR-10 dataset with IDC ($10\times32\times32$ pixels per class). Each row correspond to the condensed data of a single class. We list the class labels from the first row as follows: 1) airplane, 2) automobile, 3) bird, 4) cat, 5) deer, 6) dog, 7) frog, 8) horse, 9) ship, and 10) truck.} 
\vspace{-1em}
\label{fig:visual11}
\end{figure}

\begin{figure}[!ht]
    \centering

\begin{tikzpicture}[decoration={brace}][scale=1,every node/.style={minimum size=1cm},on grid]
    \node (svhn){\includegraphics[width=\columnwidth]{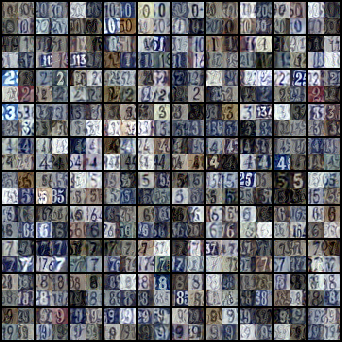}};
\end{tikzpicture}

\vspace{-1em}
\caption{Condensed images of SVHN dataset with IDC ($10\times32\times32$ pixels per class).} 
\vspace{-1em}
\label{fig:visual14}
\end{figure}

\begin{figure}[!ht]
    \centering

\begin{tikzpicture}[decoration={brace}][scale=1,every node/.style={minimum size=1cm},on grid]
    \node (imagenet){\includegraphics[width=\columnwidth]{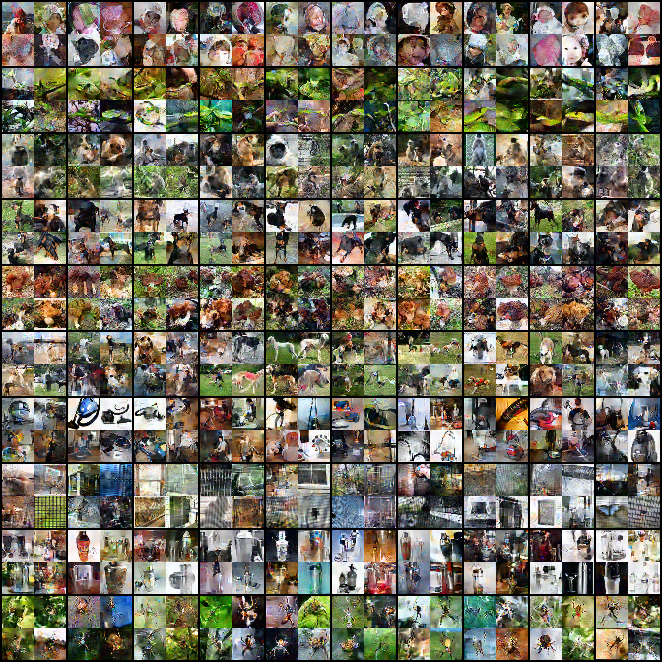}};
\end{tikzpicture}

\vspace{-1em}
\caption{Condensed images of ImageNet-10 dataset with IDC with a multi-formation factor of 2 ($10\times224\times224$ pixels per class). Each row correspond to the condensed data of a single class. We list the class labels from the first row as follows: 1) poke bonnet, 2) green mamba, 3) langur, 4) Doberman pinscher, 5) gyromitra, 6) gazelle hound, 7) vacuum cleaner, 8) window screen, 9) cocktail shaker, and 10) garden spider.} 
\vspace{-1em}
\label{fig:visual12}
\end{figure}

\end{document}